\documentclass{article}
\usepackage{iclr2027_conference,times}
\usepackage{svg}

\usepackage{amsmath,amsfonts,bm}

\def\eqref#1{(\ref{#1})}

\def\1{\bm{1}}

\DeclareMathAlphabet{\mathsfit}{\encodingdefault}{\sfdefault}{m}{sl}
\SetMathAlphabet{\mathsfit}{bold}{\encodingdefault}{\sfdefault}{bx}{n}

\usepackage{fontawesome5}
\usepackage[svgnames]{xcolor}
\usepackage[most]{tcolorbox}

\definecolor{mygreen}{HTML}{79BF41}
\definecolor{myblue}{HTML}{4DBCC9}
\definecolor{myred}{HTML}{FFB6C1} 
\definecolor{codegreen}{rgb}{0,0.6,0}
\definecolor{codegray}{rgb}{0.5,0.5,0.5}
\definecolor{codepurple}{rgb}{0.58,0,0.82}
\definecolor{backcolour}{rgb}{0.95,0.95,0.92}

\lstdefinestyle{mystyle}{
    backgroundcolor=\color{backcolour},   
    commentstyle=\color{codegreen},
    keywordstyle=\color{magenta},
    numberstyle=\tiny\color{codegray},
    stringstyle=\color{codepurple},
    basicstyle=\ttfamily\scriptsize,
    breakatwhitespace=false,         
    breaklines=true,                 
    captionpos=b,                    
    keepspaces=true,                                 
    numbersep=1pt,                  
    showspaces=false,                
    showstringspaces=false,
    showtabs=false,                  
    tabsize=1
}
\tcbset {
  base/.style={
    arc=0mm, 
    bottomtitle=0.5mm,
    boxrule=0mm,
    colbacktitle=black!10!white, 
    coltitle=black, 
    fonttitle=\bfseries, 
    left=1mm,
    leftrule=1mm,
    right=1mm,
    top=1mm,
    bottom=1mm,
    title={#1},
    toptitle=0.75mm, 
    width=\textwidth
  }
}

\newtcolorbox{greybox}[1]{
  colframe=black!15!white,
  base={#1},
  breakable
}

\newtcolorbox{promptbox}[1]{
  colframe=black!15!white,
  base={#1},
  leftrule=0mm,
  breakable,
}

\newtcolorbox{bluebox}[1]{
  colframe=myblue!50!white,
  colback=myblue!15!white,
  base={#1},
  breakable
}

\newtcolorbox{greenbox}[1]{
  colframe=mygreen!50!white,
  colback=mygreen!15!white,
  base={#1},
  breakable
}

\newtcolorbox{redbox}[1]{
  colframe=myred!50!white,
  colback=myred!15!white,
  base={#1},
  breakable
}

\usepackage{microtype}
\usepackage{graphicx}
\usepackage{booktabs} 
\usepackage{listings}

\usepackage{amsmath}
\usepackage{amssymb}
\usepackage{mathtools}
\usepackage{amsthm}
\usepackage{wrapfig}
\usepackage{booktabs}       
\usepackage{amsfonts}       
\usepackage{nicefrac}       
\usepackage{microtype}      
\usepackage{graphicx}      
\usepackage{subcaption}    

\usepackage{xcolor}
\usepackage{tikz}
\usetikzlibrary{shapes}
\usepackage{subcaption,siunitx,booktabs}

\usepackage[utf8]{inputenc} 
\usepackage[T1]{fontenc}    
\definecolor{blue1}{HTML}{2E86AB}
\usepackage[colorlinks=true,
            linkcolor=red,
            citecolor=blue1,
            filecolor=blue1,
            urlcolor=blue1]{hyperref}
\usepackage{url}

\usepackage{tcolorbox}

\usepackage{dsfont}
\usepackage{enumitem}

\usepackage[capitalize,noabbrev]{cleveref}

\theoremstyle{plain}
\newtheorem{theorem}{Theorem}[section]
\newtheorem{proposition}[theorem]{Proposition}

\theoremstyle{definition}

\newtheorem{assumption}[theorem]{Assumption}
\theoremstyle{remark}

\usepackage{minitoc}

\title{I Act Therefore I Am: When Is JEPA's Action-Conditioning Enough to Learn Causal Mechanisms?}
\author{Yuhang Liu\textsuperscript{1,2}, Zhuo Huang\textsuperscript{2}, Javen Qinfeng Shi\textsuperscript{1,2}
\\
\textsuperscript{1}Responsible AI Research Centre, Australia\\
\textsuperscript{2}Australian Institute for Machine Learning, Adelaide University\\
\texttt{yuhang.liu01@adelaide.edu.au}\\
}

\iclrfinalcopy
\begin{document}
\maketitle

\doparttoc 
\faketableofcontents

\begin{abstract}
Recent empirical and theoretical advances suggest that joint-embedding predictive architectures (JEPAs) may learn meaningful representations for action-conditioned prediction of future outcomes, thus becoming one of the foundational structures for world models. However, accurate prediction does not, in general, necessarily imply recovery of underlying causal states that give rise to the observed dynamics. This work investigates when and how JEPAs can recover the underlying causal states from observations. We first introduce a latent variable model, in which high-dimensional observations are generated from latent causal states whose dynamics are governed by action-conditioned transition mechanisms. Based on this formulation, we develop a general information-theoretic objective that combines conditional likelihood maximization for learning transition dynamics with entropy maximization for preserving latent state information. We then establish identifiability conditions under which representations learned by this general objective recover the underlying latent causal states up to component-wise invertible transformations and permutation. One key condition for such identifiability is sufficient action-induced variation in the transition mechanisms. Guided by this finding, we instantiate the general objective with an action-modulated Gaussian additive-noise model, yielding action-modulated JEPA (A-JEPA). Experiments on synthetic environments verify the theoretical findings under the identifiability conditions and robustness to moderate violations, while visual benchmarks demonstrate improved state recovery and transfer to unseen transition mechanisms.

\end{abstract}
\section{Introduction}
\label{sec: intro}
World models seek to encode observations into compact, low-dimensional representations and learn action-conditioned transition dynamics among them, thereby capturing the consequences of possible actions and supporting planning and decision-making \citep{ha2018recurrent,hafner2019learning,hafner2025mastering}. The joint-embedding predictive architecture (JEPA) provides a simple yet effective approach toward this goal by predicting future representations from past representations, conditioned on actions \citep{lecun2022path}. Empirically, this paradigm has achieved strong prediction performance \citep{assran2025v,maes2026leworldmodel,destrade2025value,huang2026vjepa}. Recent theoretical work has further provided formal guarantees for the predictive representations learned by JEPAs. In particular, under a specific Gaussian latent-variable model, JEPA can recover the underlying latent variables up to an orthogonal transformation, thereby supporting accurate future prediction \citep{klindt2026does}.

Despite recent progress in learning predictive representations, both empirically and theoretically, accurate prediction does not, in general, imply recovery of the underlying latent causal states governing the observed dynamics. In particular, representations may achieve accurate future prediction by exploiting statistical correlations without recovering the underlying causal mechanisms~\citep{pearl2009causality,scholkopf2023statistical,parascandolo2018learning,scholkopf2021toward}. Such correlations may fail to hold in previously unseen situations, such as under environmental changes, leading to degraded generalization~\citep{ye2024spurious,scholkopf2021toward}.

A natural way to address this limitation is to go beyond prediction alone and require world models to capture the underlying causal mechanisms governing the observed dynamics, giving rise to causal world models \citep{richens2024robust,zhu2022offline,ceriscioli2026agents}. Recent work has begun to explore this direction within the JEPA framework \citep{nam2026causal}. However, this approach relies primarily on heuristically imposed causal inductive biases, and it still urgently needs formal guarantees for recovering the underlying causal states from learned representations. This work therefore asks the following question:
\begin{center}
\begin{greybox}{}
When and how can JEPA recover the underlying causal states from observational data?
\vskip -0.1in
\end{greybox}
\label{key_question}
\end{center}

To answer the question, this work makes the following contributions. 
\begin{itemize}[leftmargin=*]
    \item We introduce a latent variable model to formulate JEPA in Sec.~\ref{sec:setting}, in which observations are generated from latent causal states evolving through action-conditioned transition mechanisms. Based on this, we develop a general information-theoretic objective that combines conditional likelihood maximization with entropy maximization for learning predictive and information-preserving representations.
    \item Building on the formulation above, we establish an identifiability theory in Sec.~\ref{sec:identifiability}, showing that, under certain assumptions, representations learned by the proposed general objective recover the latent causal states up to component-wise invertible transformations and permutation. \textit{To the best of our knowledge, this is the first component-wise identifiability result within the JEPA framework.}
    \item Guided by the theoretical finding, we instantiate the general objective with an action-modulated Gaussian additive-noise model, yielding action-modulated JEPA (A-JEPA) in Sec.~\ref{sec:case-study}.
    \item We conduct extensive experiments in Sec.~\ref{sec:exp} on simulated data, verifying the theoretical findings when the identifiability assumptions hold and demonstrating robustness under moderate violations. Experiments on visual benchmarks further show substantially improved component-wise latent-state recovery and strong transferability to unseen transition mechanisms.
\end{itemize}

\section{A Latent Variable Model and an Information-Theoretic JEPA}
\label{sec:setting}
To study when and how JEPA can recover the underlying latent causal states, we first formulate the observed action-conditioned dynamics through a latent variable model. Under this formulation, we examine what a JEPA objective requires to support latent-state recovery, motivating a general information-theoretic objective that combines transition predictiveness with information preservation. We then characterize the global optima of this objective, providing the foundation for the identifiability analysis in Sec.~\ref{sec:identifiability}.

\subsection{Setup: A Latent Causal Generative Model}
\label{sec:latent-model}

Consider a latent causal generative process indexed by time $t$, with latent state $\mathbf s_t \in \mathbb R^d$, observed action $\mathbf a_t$, and high-dimensional observation $\mathbf x_t$. We assume that the latent dynamics are governed by a temporal structural causal model (SCM) whose associated graph is a directed acyclic graph, such that each latent variable $s_{t,i}$ depends only on a subset of variables in the previous state. Let $\mathrm{pa}_i \subseteq \{1,\ldots,d\}$ denote the indices of the causal parents of $s_{t,i}$ in $\mathbf s_{t-1}$. The action $\mathbf a_{t-1}$ is treated as an observed exogenous input that may modulate each transition mechanism.

We model the latent causal dynamics and observation process as:
\begin{align}
s_{t,i} &:= f_i\left( \mathbf s_{t-1,\mathrm{pa}_i}, \mathbf a_{t-1}, \epsilon_{t,i} \right), \quad i=1,\ldots,d,
\label{eq:latent-scm} \\
\mathbf x_t &:= \boldsymbol{\mathrm g}(\mathbf s_t).
\label{eq:observation-model}
\end{align}
Here, $f_i$ denotes the structural causal mechanism governing the transition of the $i$-th latent variable, $\epsilon_{t,i}$ is the corresponding exogenous disturbance, and $\boldsymbol{\mathrm g}$ is an unknown deterministic observation mapping. Together, Eqs.~\eqref{eq:latent-scm} and~\eqref{eq:observation-model} define the latent causal generative process underlying the observed action-conditioned dynamics.

We assume that the exogenous disturbances $\epsilon_{t,1},\ldots,\epsilon_{t,d}$ are mutually independent and independent of $(\mathbf s_{t-1},\mathbf a_{t-1})$, following the standard independent-noise assumption in structural causal models~\citep{pearl2009causality,peters2017elements,scholkopf2021toward}.
The temporal SCM in Eq.~\eqref{eq:latent-scm} therefore induces the action-conditioned transition factorization:
\begin{equation}
p(\mathbf s_t \mid \mathbf s_{t-1}, \mathbf a_{t-1}) = \prod\nolimits_{i=1}^{d} p\left(s_{t,i} \mid \mathbf s_{t-1,\mathrm{pa}_i}, \mathbf a_{t-1} \right).
\label{eq:causal-transition-factorization}
\end{equation}

\subsection{A General Information-Theoretic JEPA Objective}
\label{sec:information-jepa}
Given the latent causal generative model in Eqs.~\eqref{eq:latent-scm} and~\eqref{eq:observation-model}, we investigate what a JEPA objective needs to learn representations that recover the underlying latent causal states $\mathbf s_t$. Conventional JEPA objectives primarily encourage prediction of the future representation from the previous representation and action. Intuitively, beyond supporting future prediction, latent-state recovery additionally requires the learned representation to retain sufficient information about the underlying latent state. This motivates two complementary principles:
\begin{itemize}[leftmargin=*]
    \item \textit{Transition Predictiveness.} The learned representation should support accurate prediction of the future representation from the previous representation and action.
    \item \textit{Information Preservation.} The learned representation should avoid discarding information about the underlying latent state that is necessary for its recovery.
\end{itemize}
We instantiate these two principles through conditional likelihood maximization and entropy maximization, respectively. Let
$\mathbf z_t=\boldsymbol{\mathrm h}(\mathbf x_t)$ denote the representation produced by an encoder $\boldsymbol{\mathrm h}$, and let
$p_{\boldsymbol\phi}(\mathbf z_t\mid\mathbf z_{t-1},\mathbf a_{t-1})$
denote a parameterized action-conditioned transition model. As a result, we arrive at the following general information-theoretic JEPA objective:
\begin{equation}
\mathcal L_{\lambda}(\boldsymbol{\mathrm{h}},\boldsymbol{\mathrm{\boldsymbol{\mathrm{\phi}}}}) = -\mathbb E \left[ \log p_{\boldsymbol{\mathrm{\boldsymbol{\mathrm{\phi}}}}} (\mathbf z_t\mid\mathbf z_{t-1},\mathbf a_{t-1}) \right] -\lambda H(\mathbf z_t),
\label{eq:information-jepa-objective}
\end{equation}
where $H(\mathbf z_t)$ denotes the differential entropy of the learned representation and $\lambda$ controls the trade-off between transition predictiveness and information preservation. The first term encourages the learned transition model to accurately characterize the action-conditioned prediction in representation space, while the second discourages the learned representation from discarding information about the underlying latent state by favoring high-entropy representations.

We next provide an initial theoretical analysis of the objective in Eq.~\eqref{eq:information-jepa-objective} to examine the properties it provides toward latent state recovery. For a given encoder $\boldsymbol{\mathrm h}$, the latent generative model together with the encoder induces an action-conditioned transition distribution in the representation space, which we denote by $Q_{\boldsymbol{\mathrm h}}(\mathbf z_t\mid\mathbf z_{t-1},\mathbf a_{t-1})$. In contrast, $p_{\boldsymbol{\phi}}(\mathbf z_t\mid\mathbf z_{t-1},\mathbf a_{t-1})$ denotes the parameterized transition model optimized to approximate this induced distribution. With these definitions, we obtain the following result.

\begin{theorem}
\label{thm:optimal-solution}
Suppose the observations are generated according to the latent causal generative model in Eqs.~\eqref{eq:latent-scm}-\eqref{eq:observation-model}. Let $\mathbf z_t=\boldsymbol{\mathrm h}(\mathbf x_t)$ with $\boldsymbol{\mathrm h}:\mathcal X\rightarrow(0,1)^d$, and let $(\boldsymbol{\mathrm h}^{\star},\boldsymbol{\phi}^{\star})$ be a global
minimizer of Eq.~\eqref{eq:information-jepa-objective} with $\lambda\geq1$ that attains its lower bound. Then
\begin{align}
I(\mathbf z_t;\mathbf z_{t-1},\mathbf a_{t-1}) &= I(\mathbf s_t;\mathbf s_{t-1},\mathbf a_{t-1}), \label{eq:optimal-info}\\
p_{\boldsymbol{\phi}^{\star}} (\mathbf z_t\mid\mathbf z_{t-1},\mathbf a_{t-1}) &= Q_{\boldsymbol{\mathrm h}^{\star}}(\mathbf z_t\mid\mathbf z_{t-1},\mathbf a_{t-1})
\quad \mathrm{a.e.}
\label{eq:optimal-transition}
\end{align}
\end{theorem}
\begin{proof}
    See Appendix~\ref{app:proof-optimal-solution}.
\end{proof}

\paragraph{Potential for Latent-State Recovery.}
Theorem~\ref{thm:optimal-solution} provides two key characterizations of the objective in Eq.~\eqref{eq:information-jepa-objective} at a global optimum. First, the learned representation preserves all transition-predictive information contained in the underlying latent state, as characterized by Eq.~\eqref{eq:optimal-info}. This ensures that the encoder $\boldsymbol h$ retains the information in the latent state that is relevant for predicting the transition dynamics. Second, the learned transition model matches the transition distribution induced in the representation space, as shown by Eq.~\eqref{eq:optimal-transition}. This ensures that $\boldsymbol{\phi}$ reproduces the induced action-conditioned transition distribution in the learned representation space. Together, these two properties ensure that the learned representation retains the information required to describe the latent transition, while the learned transition model faithfully reproduces the corresponding conditional distribution in representation space.

However, these properties only characterize the learned representation at the level of transition-relevant information preservation and conditional distribution matching, and do not by themselves guarantee component-wise recovery of the individual latent causal states. That is, the learned representation may still correspond to non-trivial mixtures of the underlying latent causal states, e.g., through an orthogonal transformation. We therefore study additional conditions under which such mixing can be ruled out in the next section.
\section{Identifiability Analysis: Answering When}
\label{sec:identifiability}
In this section, we study conditions under which the learned representation recovers the latent causal states up to component-wise invertible transformations and permutation, i.e., component-wise identifiability. Without additional assumptions, such identifiability may not in general be guaranteed. We introduce the following conditions on the latent causal generative model in Eqs.~\eqref{eq:latent-scm} and~\eqref{eq:observation-model}.

\begin{assumption} 
\label{assump:regularity}
The mapping $\boldsymbol{\mathrm g}$ is smooth and invertible. 
\end{assumption}
Assumption~\ref{assump:regularity} is to ensure that the observation process Eq.~\ref{eq:observation-model} does not discard information about the latent state, so that the latent state can, in principle, be recovered from the observation. Intuitively, if the observation mapping is non-invertible, distinct latent states may generate the same observation, making exact latent state recovery impossible in general.

\begin{assumption}
\label{assump:action-variability}
For any $\mathbf s_t$, there exist $2d+1$ actions $\mathbf a_{t-1}^{(0)},\ldots,\mathbf a_{t-1}^{(2d)}$ such that the matrix
\begin{equation}
\mathbf L = \left[ \mathbf w(\mathbf s_t,\mathbf s_{t-1},\mathbf a_{t-1}^{(1)}) - \mathbf w(\mathbf s_t,\mathbf s_{t-1},\mathbf a_{t-1}^{(0)}), \ldots, \mathbf w(\mathbf s_t,\mathbf s_{t-1},\mathbf a_{t-1}^{(2d)}) - \mathbf w(\mathbf s_t,\mathbf s_{t-1},\mathbf a_{t-1}^{(0)}) \right]
\end{equation}
is invertible. Here, for each $i$, $ q_i=\log p\!\left(s_{t,i} \mid \mathbf s_{t-1,\mathrm{pa}_i}, \mathbf a_{t-1} \right)$, and
\begin{equation}
\mathbf w(\mathbf s_t,\mathbf s_{t-1},\mathbf a_{t-1})
:=
\left(
\frac{\partial q_1}{\partial s_{t,1}},
\ldots,
\frac{\partial q_d}{\partial s_{t,d}},
\frac{\partial^2 q_1}{\partial s_{t,1}^2},
\ldots,
\frac{\partial^2 q_d}{\partial s_{t,d}^2}
\right)^\top .
\end{equation}
\end{assumption}
Intuitively, Assumption~\ref{assump:action-variability} requires different actions to induce sufficiently diverse changes in the latent transition mechanisms. Such action-induced variation provides distinct signatures for individual latent causal variables, allowing them to be distinguished from arbitrary mixtures.

With the main assumptions above, we can now establish the main identifiability result, as follows:
\begin{theorem}
\label{thm:componentwise-identifiability}
Suppose the latent causal states $\mathbf s_t$ and observations $\mathbf x_t$ follow the generative model defined in Eqs.~\eqref{eq:latent-scm} and~\eqref{eq:observation-model}. Assume that Assumptions~\ref{assump:regularity} and-\ref{assump:action-variability} hold. Let $\boldsymbol{\mathrm h}:\mathcal X\rightarrow(0,1)^d$ be any smooth
encoder associated with a global minimizer of
Eq.~\eqref{eq:information-jepa-objective} with $\lambda>1$. Then the true latent causal states $\mathbf s_t$ are related to the learned representations $\mathbf z_t$, which are obtained by the encoder $\boldsymbol{\mathrm h}$, by the following relationship: 
\begin{equation}
    \mathbf z_t = \left( r_1(s_{t,\pi(1)}), \ldots, r_d(s_{t,\pi(d)}) \right),
\end{equation}
where $\pi$ is a permutation and $r_1,\ldots,r_d$ are one-dimensional invertible transformations.
\end{theorem}
\begin{proof}
    See Appendix~\ref{app:componentwise-identifiability}.
\end{proof}

\paragraph{Implication: Answering When.}
Theorem~\ref{thm:componentwise-identifiability} shows that JEPA can recover the latent causal states up to component-wise invertible transformations and permutation. A key condition underlying this result is sufficient action variability, i.e., Assumption~\ref{assump:action-variability}: different actions must induce sufficiently diverse changes in the latent transition mechanisms to rule out arbitrary mixing of the latent causal variables. This provides a primary answer to the ``when'' part of our central question: JEPA can identify the underlying latent causal states when the observed actions induce sufficiently rich variation in the transition mechanisms. Intuitively, if the transition mechanisms remain unchanged across actions, different latent mixtures may remain observationally indistinguishable, making component-wise recovery difficult.

\section{From Identifiability to Practice: Answering How}
\label{sec:case-study}

Theorem~\ref{thm:componentwise-identifiability} identifies sufficient action-induced variation in the latent transition mechanisms as a key condition for component-wise recovery. While the result applies to general nonparametric transition mechanisms, it does not directly specify how such variation should be modeled and optimized in practice. We therefore consider a concrete transition model in which action-induced mechanism variation admits an explicit and tractable characterization.

To this end, we consider an action-modulated additive-noise model (ANM) with Gaussian noise as a representative instantiation. This choice is motivated by three considerations. First, ANMs provide a widely used class for modeling causal mechanisms~\citep{hoyer2008nonlinear,JMLR:v15:peters14a,buhlmann2014cam}. Second, the Gaussian formulation yields a tractable conditional likelihood. Most importantly, it makes the action-variability condition in Assumption~\ref{assump:action-variability} directly interpretable through action-dependent conditional means and variances, as discussed below.

\paragraph{Action-Modulated Gaussian ANM in Latent Space.} We instantiate the latent causal transition model in
Eq.~\eqref{eq:causal-transition-factorization} as an action-modulated Gaussian ANM, as follows:
\begin{equation}
s_{t,i}
=
f_i(\mathbf s_{t-1,\mathrm{pa}_i},\mathbf a_{t-1})
+
\sigma_i(\mathbf a_{t-1})
\epsilon_{t,i},
\qquad
\epsilon_{t,i}\sim\mathcal N(0,1),
\label{eq:gaussian-anm}
\end{equation}
where $f_i$ characterizes the deterministic transition mechanism and
$\sigma_i$ models the corresponding transition uncertainty. Both are allowed to vary with the action $\mathbf a_{t-1}$. Accordingly, the conditional transition distribution takes the form:
\begin{equation}
p(\mathbf s_t\mid\mathbf s_{t-1},\mathbf a_{t-1})=\prod\nolimits_{i=1}^{d}\mathcal N\left(s_{t,i};
f_i(\mathbf s_{t-1,\mathrm{pa}_i},\mathbf a_{t-1}),
\sigma_i^2(\mathbf a_{t-1})
\right).
\label{eq:gaussian-factorization}
\end{equation}
Under this formulation, action-induced changes in the transition mechanisms are directly reflected by the changes in the conditional means $f_i$ and variances $\sigma_i^2$. For this Gaussian model, the first- and second-order
derivatives of the conditional log-density appearing in Assumption~\ref{assump:action-variability} take the form
\begin{equation}
\frac{\partial q_i}{\partial s_{t,i}} = -\frac{s_{t,i}-f_i}{\sigma_i^2}, \qquad \frac{\partial^2 q_i}{\partial s_{t,i}^2} = -\frac{1}{\sigma_i^2},
\label{eq:gaussian-score}
\end{equation}
where the dependence of $f_i$ on $(\mathbf s_{t-1,\mathrm{pa}_i},\mathbf a_{t-1})$ and of $\sigma_i^2$ on $\mathbf a_{t-1}$ is omitted for clarity. Eq.~\eqref{eq:gaussian-score} provides a direct interpretation of Assumption~\ref{assump:action-variability}. Action-induced changes in the conditional mean $f_i$ affect the first-order derivative, whereas changes in the conditional variance $\sigma_i^2$ affect both the first- and second-order derivatives. Consequently, sufficient action variability requires actions to induce sufficiently rich variation in these conditional transition statistics so that the resulting matrix $\mathbf L$ has full rank.

\paragraph{Gaussian Transition Model in Representation Space.}
Following the Gaussian instantiation in the latent space, we parameterize
the learned transition model to preserve the component-wise temporal
causal structure. Specifically, let $\mathrm{pa}_i$ denote the parent
set of the $i$-th representation component in the previous state. We
model
\begin{equation}
p_{\boldsymbol{\phi}}
(\mathbf z_t\mid\mathbf z_{t-1},\mathbf a_{t-1})
=
\prod\nolimits_{i=1}^{d}
\mathcal N\left(
z_{t,i};
\mu_{\boldsymbol{\phi},i}
(\mathbf z_{t-1,\mathrm{pa}_i},\mathbf a_{t-1}),
\sigma_{\boldsymbol{\phi},i}^2
(\mathbf a_{t-1})
\right).
\label{eq:learned-gaussian-transition}
\end{equation}
Here,
$\mu_{\boldsymbol{\phi},i}
(\mathbf z_{t-1,\mathrm{pa}_i},\mathbf a_{t-1})$
and
$\sigma_{\boldsymbol{\phi},i}
(\mathbf a_{t-1})>0$
denote the conditional mean and standard deviation of the $i$-th
representation component, respectively. This parameterization mirrors the
latent transition model by allowing each representation component to
depend only on its temporal parents and the action. Implementations are detailed in Appendix~\ref{sec:graph-implementation}.

The corresponding negative log-likelihood (NLL) term in
Eq.~\eqref{eq:information-jepa-objective} is
\begin{equation}
\begin{aligned}
\mathrm{NLL}_{\boldsymbol{\phi}}
(\mathbf z_t\mid\mathbf z_{t-1},\mathbf a_{t-1})
=
\frac{1}{2}\sum_{i=1}^{d}
\Bigg[
\frac{
\left(
z_{t,i}
-
\mu_{\boldsymbol{\phi},i}
(\mathbf z_{t-1,\mathrm{pa}_i},\mathbf a_{t-1})
\right)^2
}{
\sigma_{\boldsymbol{\phi},i}^2
(\mathbf a_{t-1})
}+
\log
\sigma_{\boldsymbol{\phi},i}^2
(\mathbf a_{t-1})
\Bigg].
\end{aligned}
\label{eq:gaussian-nll}
\end{equation}
Here, we omit the additive constant
$d\log(2\pi)/2$, which does not affect optimization.

The remaining challenge is to instantiate the entropy term in
Eq.~\eqref{eq:information-jepa-objective}. Direct estimation of
differential entropy in high-dimensional representation spaces is
generally difficult. We therefore adopt a contrastive surrogate for the
entropy term, motivated by the established asymptotic connection between
contrastive learning with negative samples and entropy-regularized
representation learning~\citep{wang2020understanding,von2021self}.
For a theoretical justification of this surrogate, see
Sec.~\ref{app:contrastive-entropy}.

Specifically, given an observed transition
$(\mathbf z_{t-1},\mathbf a_{t-1},\mathbf z_t)$, we sample $K$
alternative future representations
$\mathbf z_t^{(1)},\ldots,\mathbf z_t^{(K)}$
from other transitions in the minibatch and evaluate them under the same
action-conditioned transition model. Combining the NLL in
Eq.~\eqref{eq:gaussian-nll} with this contrastive surrogate yields the
practical objective
\begin{equation}
\small
\mathcal L
=
\mathbb E\left[
\alpha\,
\mathrm{NLL}_{\boldsymbol{\phi}}
(\mathbf z_t\mid\mathbf z_{t-1},\mathbf a_{t-1})
+
(1-\alpha)
\log\frac{1}{K}
\sum_{k=1}^{K}
\exp\left(
-\frac{
\mathrm{NLL}_{\boldsymbol{\phi}}
(\mathbf z_t^{(k)}
\mid\mathbf z_{t-1},\mathbf a_{t-1})
}{\tau}
\right)
\right].
\label{eq:gaussian-contrastive-loss}
\end{equation}
Here, $\alpha\in(0,1)$ controls the trade-off between transition
prediction and information preservation, while $\tau>0$ denotes the
temperature. The first term assigns high likelihood to the future representation, whereas the second discourages the same transition model from assigning high likelihood to alternative future representations. Together, the two terms encourage representations that are both transition-predictive and information-preserving, thus fulfilling our JEPA objective principles.

\paragraph{Implication: Answering How.}
Eq.~\eqref{eq:gaussian-contrastive-loss} provides a practical instantiation of the two principles underlying Eq.~\eqref{eq:information-jepa-objective}. The Gaussian NLL implements transition predictiveness through conditional likelihood maximization, while the contrastive term provides a tractable surrogate for entropy-based information preservation. More importantly, the action-modulated Gaussian transition translates the abstract variability condition in Assumption~\ref{assump:action-variability} into a concrete modeling principle: by parameterizing action-dependent conditional means and variances, the learned transition model can capture the action-induced mechanism variation to ensure component-wise identifiability.
\section{Empirical Evaluation}
\label{sec:exp}
We first evaluate A-JEPA on synthetic systems generated according to the action-modulated latent transition model in Sec.~\ref{sec:case-study}, where the theoretical assumptions are satisfied by construction, to empirically examine the identifiability results developed above. We then consider controlled violations of these assumptions to assess how the method behaves when the theoretical conditions are only approximately satisfied. Finally, we evaluate on visual control environments to study whether the same design facilitates recovery of task-relevant state factors.

\subsection{Simulation}
\label{sec:simulation}

\paragraph{Data generation.} We generate synthetic temporal causal systems according to the action-modulated Gaussian ANM in Sec.~\ref{sec:case-study}, with ground-truth latent states and temporal DAGs. The setup provides controlled variation along the key dimensions of our analysis: latent dimensionality, action-induced mechanism variation, observation invertibility, and transition-noise specification. Actions modulate both the conditional means and variances of the latent transition mechanisms, while latent states are mapped to observations through an invertible nonlinear mixing function in the matched setting. This controlled setup allows us to evaluate component-wise state recovery, transition-structure recovery, and robustness to violations of the identifiability and modeling assumptions. See Sec.~\ref{app:simulation-details} for details.

\paragraph{Model and Evaluation.} We train the encoder and action-conditioned transition model using the objective in Eq.~\eqref{eq:gaussian-contrastive-loss}. Consistent with the component-wise identifiability result in Theorem~\ref{thm:componentwise-identifiability}, we evaluate latent-state recovery using Spearman component-wise mean correlation coefficient (MCC). Specifically, we compute absolute Spearman correlations between learned and ground-truth latent coordinates, perform optimal matching using the Hungarian algorithm, and average the matched correlations. For transition-structure recovery, we align the learned coordinates using the same matching and evaluate the recovered temporal graph using edge F1, precision, recall, and structural Hamming distance (SHD).


\begin{figure}[h]
    \centering
     \vspace{-5pt}
    \begin{subfigure}[t]{0.32\textwidth}
        \centering
        \includegraphics[width=\linewidth]{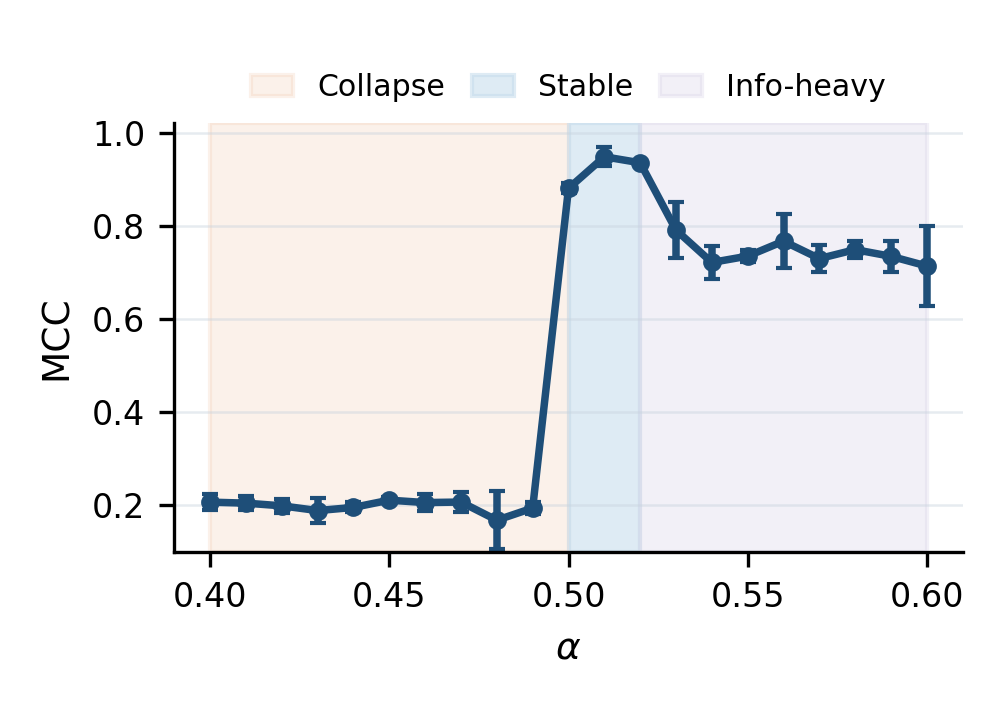}
        \caption{Trade-off hyperparameter $\alpha$.}
        \label{fig:sim-alpha-sweep}
    \end{subfigure}
    \hfill
    \begin{subfigure}[t]{0.32\textwidth}
        \centering
        \includegraphics[width=\linewidth]{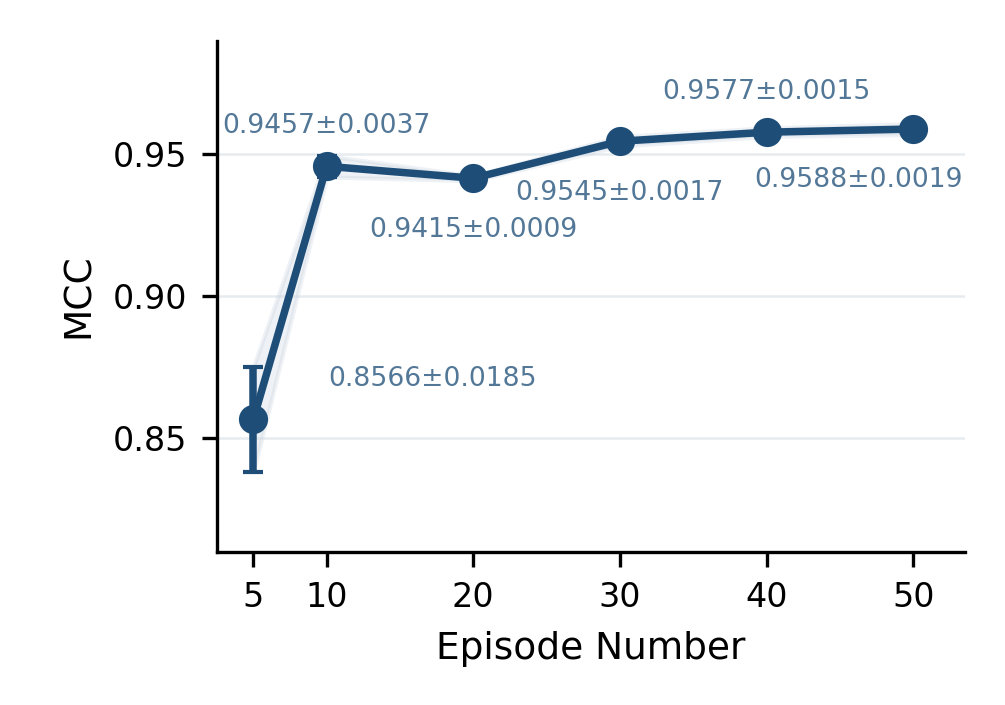}
        \caption{Action variability.}
        \label{fig:sim-action-variability}
    \end{subfigure}
    \hfill
    \begin{subfigure}[t]{0.32\textwidth}
        \centering
        \includegraphics[width=\linewidth]{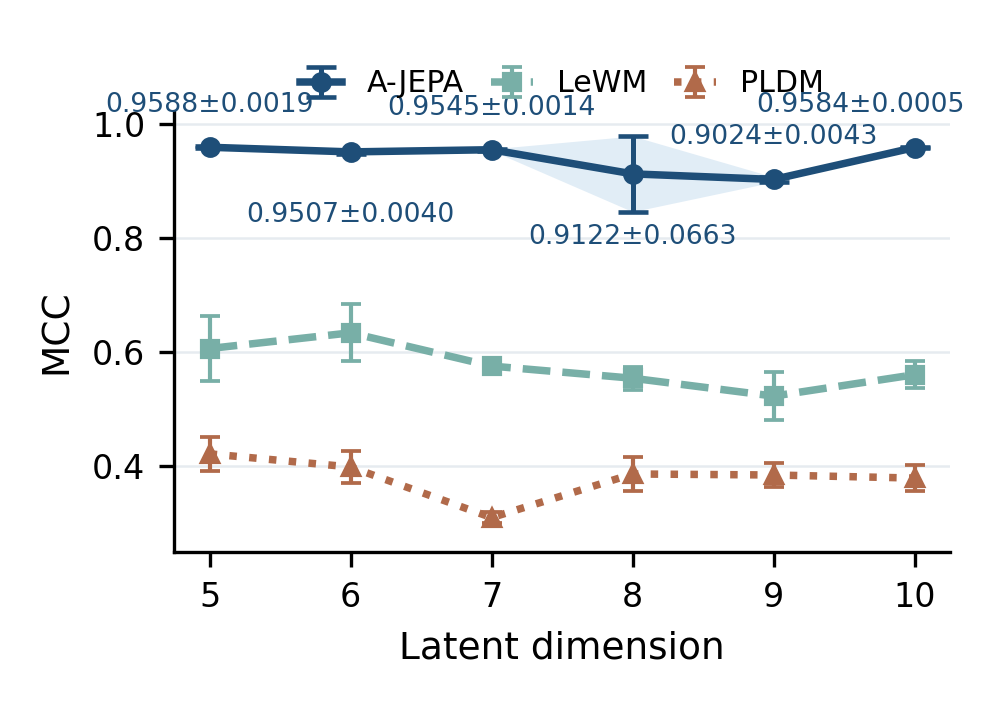}
        \caption{Latent dimension scaling.}
        \label{fig:sim-latent-scaling}
    \end{subfigure}
    \caption{Analysis of latent-state recovery. Left: effect of the trade-off between transition likelihood and information preservation. Middle: effect of increasing action-induced mechanism variation. Right: component-wise recovery across latent dimensions compared with JEPA-style baselines.}
       \vspace{-10pt}
    \label{fig:sim-main-analysis}
\end{figure}

\paragraph{Latent-State Recovery Analysis.} We first examine the factors governing latent-state recovery. Varying the trade-off parameter $\alpha$ in Eq.~\eqref{eq:gaussian-contrastive-loss} reveals a sharp dependence on the balance between transition likelihood and information preservation (Fig.~\ref{fig:sim-alpha-sweep}). MCC remains low for $\alpha<0.50$, peaks at $0.9479$ near $\alpha=0.51$, and then gradually decreases as greater weight is placed on transition likelihood. This highlights the importance of balancing transition predictiveness with information preservation for component-wise recovery. We next vary the number of distinct action settings. As shown in Fig.~\ref{fig:sim-action-variability}, MCC improves from $0.8566$ with five actions to $0.9588$ with fifty, with gains gradually saturating. This trend is consistent with Assumption~\ref{assump:action-variability}: richer action-induced mechanism variation provides stronger signals for distinguishing latent causal variables. Finally, Fig.~\ref{fig:sim-latent-scaling} evaluates recovery as the latent dimension increases from $d=5$ to $d=10$. A-JEPA maintains MCC around $0.90$-$0.96$, substantially outperforming LeWM-style and PLDM baselines. This shows that strong component-wise recovery persists as dimensionality increases.

\begin{figure}[h]
    \centering
    \begin{subfigure}[t]{0.32\textwidth}
        \centering
        \includegraphics[width=\linewidth]
        {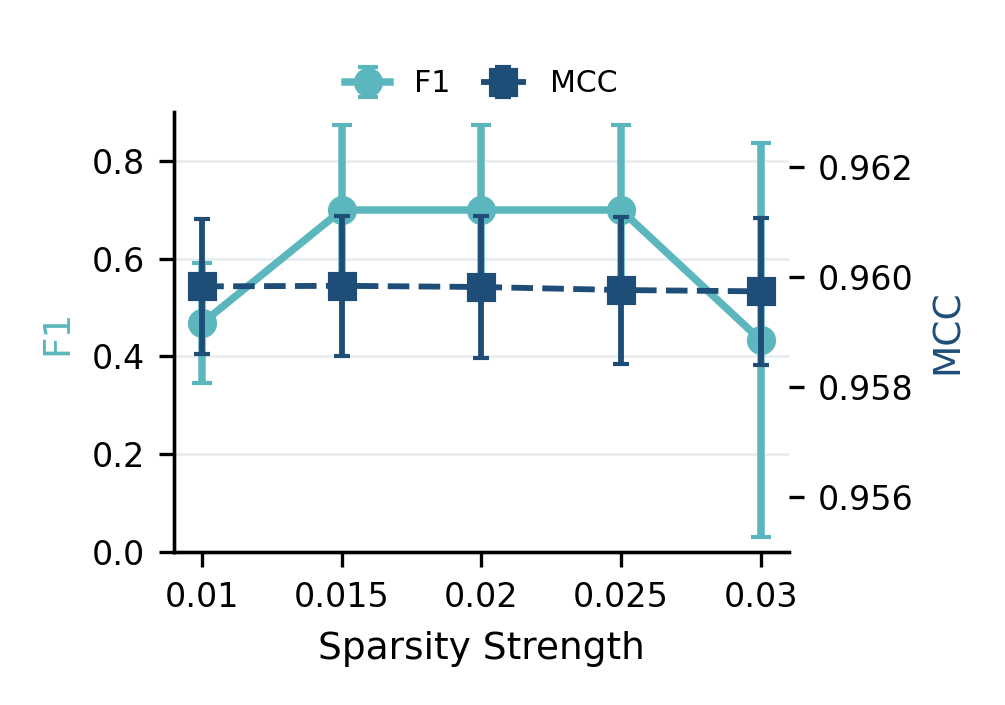}
        \caption{Parent-gate sparsity.}
        \label{fig:sim-parent-gate}
    \end{subfigure}
    \hfill
    \begin{subfigure}[t]{0.32\textwidth}
        \centering
        \includegraphics[width=\linewidth]
        {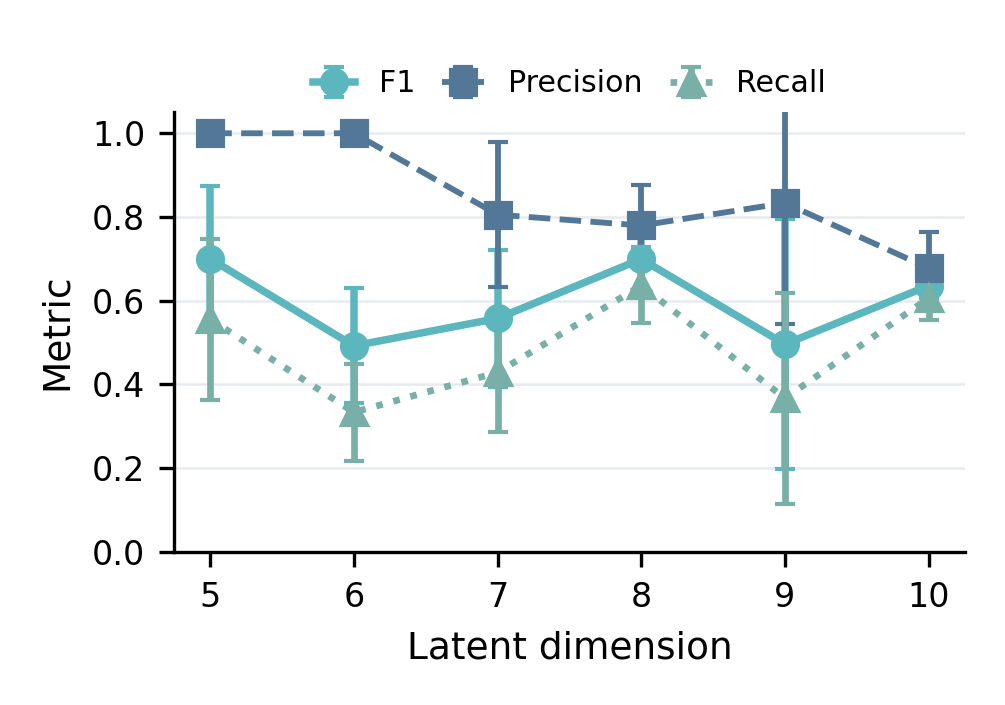}
        \caption{Graph recovery.}
        \label{fig:sim-graph-metrics}
    \end{subfigure}
    \hfill
    \begin{subfigure}[t]{0.32\textwidth}
        \centering
        \includegraphics[width=\linewidth]
        {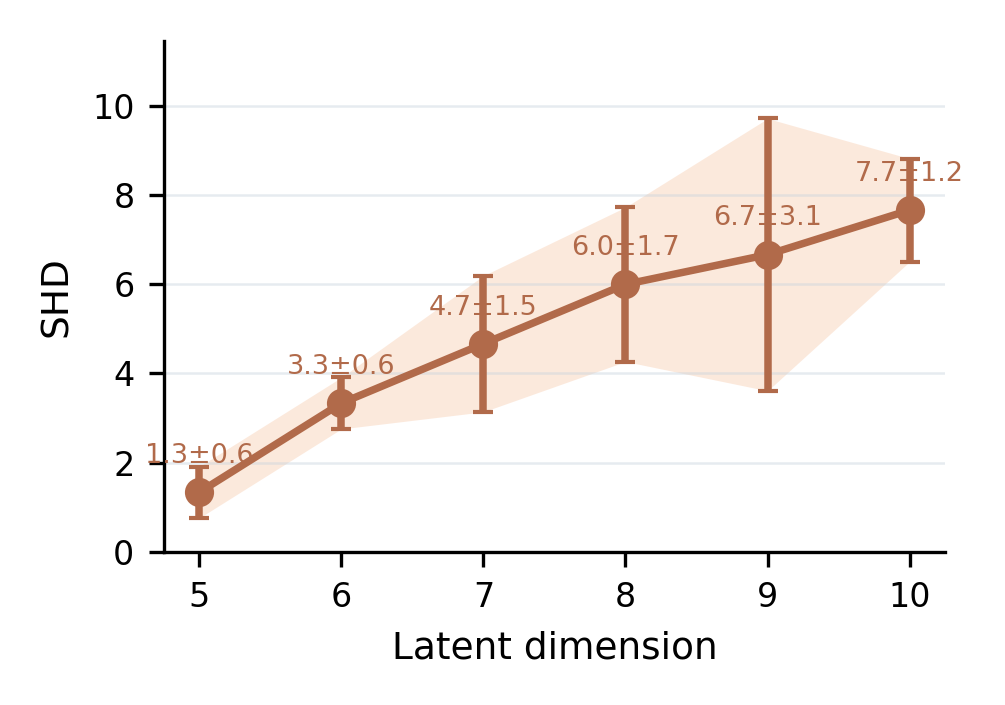}
        \caption{Structural Hamming distance.}
        \label{fig:sim-graph-shd}
    \end{subfigure}
    \caption{Analysis of latent structure recovery. Left: effect of parent-gate sparsity on graph recovery. Middle: F1, precision, and recall across latent dimensions. Right: SHD across latent dimensions.}
       \vspace{-10pt}
    \label{fig:sim-structure-recovery}
\end{figure}

\paragraph{Transition-Structure Recovery Analysis.} We further examine whether the recovered latent representations support identification of the underlying causal transition structure. As shown in Fig.~\ref{fig:sim-parent-gate}, latent-state recovery remains stable at MCC around $0.96$, whereas graph recovery is more sensitive to the gate sparsity regularization (See Eq.~\eqref{eq:sparse-ajepa-objective} for implementation.), with edge F1 reaching approximately $0.70$ in the best-performing range. We then evaluate graph recovery as the latent dimension increases from $d=5$ to $d=10$. Fig.~\ref{fig:sim-graph-metrics} shows consistently high precision with moderate recall, while Fig.~\ref{fig:sim-graph-shd} shows increasing SHD as dimensionality grows. Together, these results indicate that the recovered representations retain meaningful information about the underlying transition structure, although graph recovery becomes increasingly challenging in higher dimensions.

\begin{figure}[h]
    \centering
    \begin{subfigure}[t]{0.32\textwidth}
        \centering
        \includegraphics[width=\linewidth]
        {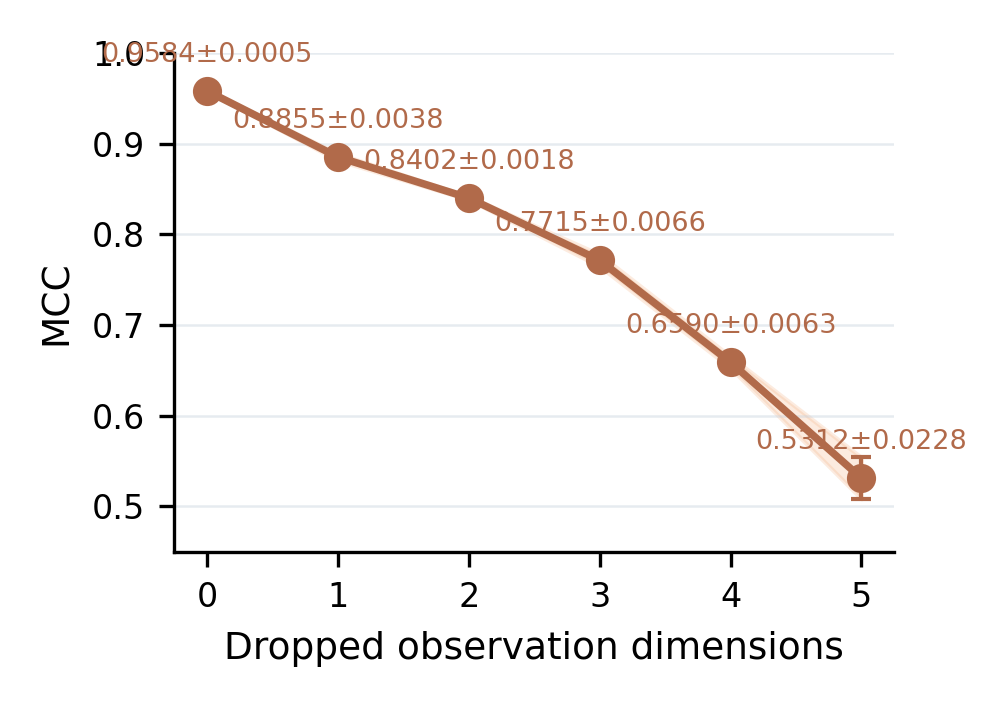}
        \caption{Observation invertibility.}
        \label{fig:sim-invertibility}
    \end{subfigure}
    \hfill
    \begin{subfigure}[t]{0.32\textwidth}
        \centering
        \includegraphics[width=\linewidth]
        {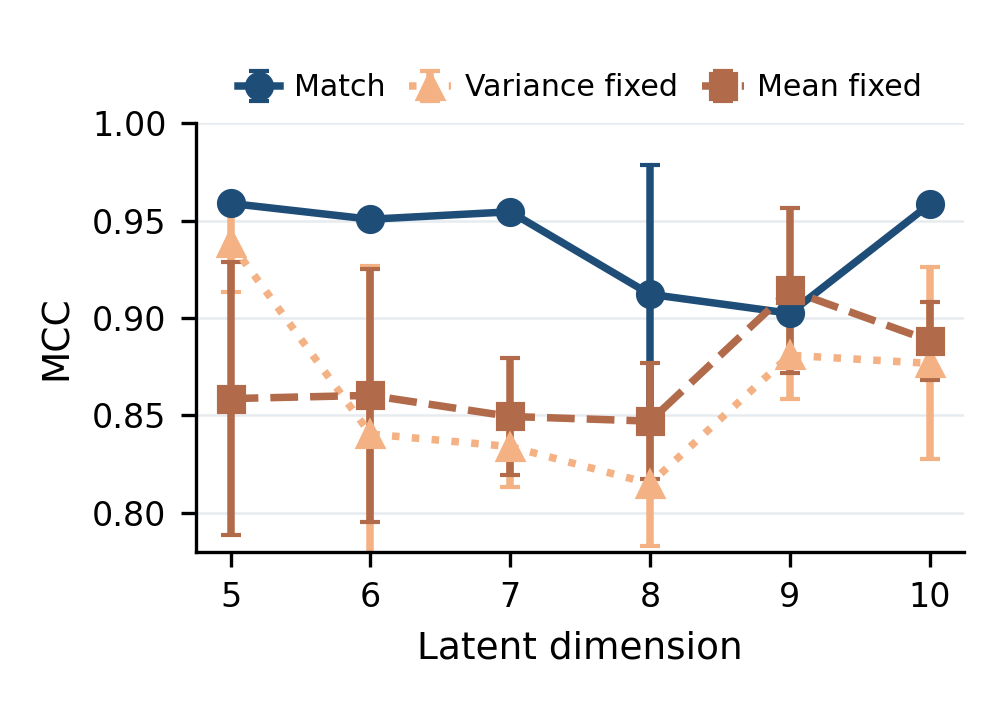}
        \caption{Action-induced variation.}
        \label{fig:sim-mechanism-violation}
    \end{subfigure}
    \hfill
    \begin{subfigure}[t]{0.32\textwidth}
        \centering
        \includegraphics[width=\linewidth]
        {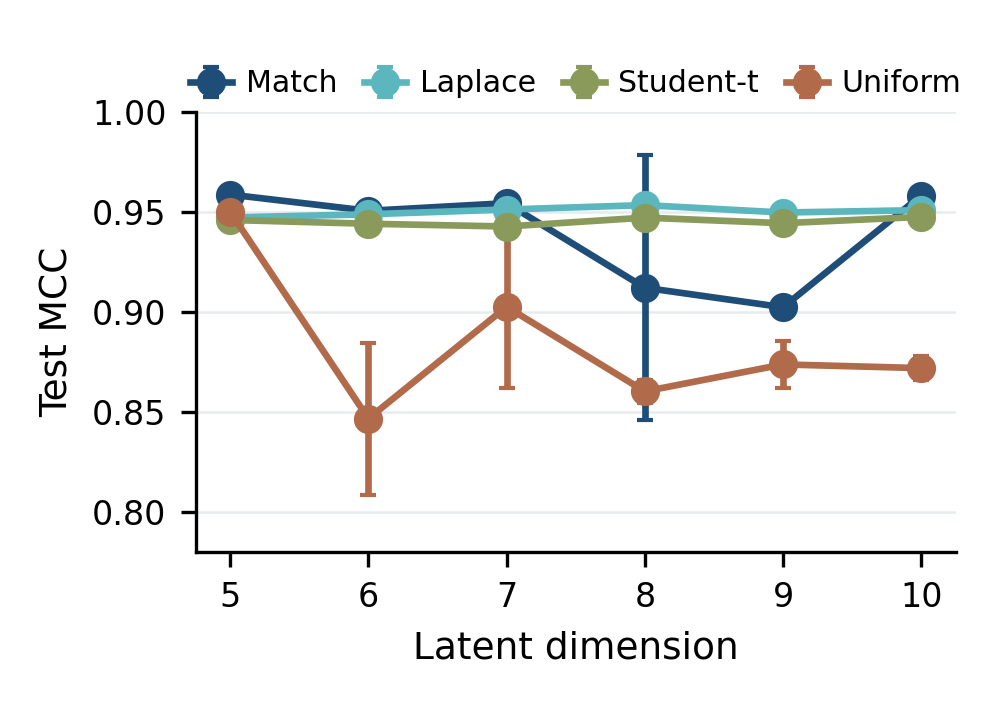}
        \caption{Noise misspecification.}
        \label{fig:sim-noise}
    \end{subfigure}

    \caption{Assumption and model-specification analysis. 
    Left: latent-state recovery degrades as the observation mapping becomes
    increasingly non-invertible. Middle: reducing action-induced variation in
    either the conditional mean or variance degrades recovery. Right: the
    proposed method remains robust to Laplace and Student-$t$ noise, while
    uniform noise leads to a larger degradation.}
    \label{fig:sim-assumption-violations}
\end{figure}

\paragraph{Assumption and Model-Specification Analysis.}
Finally, we examine sensitivity to the assumptions and modeling choices underlying our framework. First, progressively removing observed coordinates degrades MCC from $0.9584$ to $0.5312$ (Fig.~\ref{fig:sim-invertibility}), consistent with the importance of an information-preserving observation mapping. Second, we weaken action-induced mechanism variation by allowing actions to modulate only the conditional mean or variance. Both restrictions degrade latent-state recovery across dimensions (Fig.~\ref{fig:sim-mechanism-violation}), supporting the role of rich action-induced variation in Assumption~\ref{assump:action-variability}. Finally, replacing Gaussian noise with Laplace or Student-$t$ noise causes little degradation, whereas uniform noise has a larger effect (Fig.~\ref{fig:sim-noise}), indicating robustness to moderate transition-noise misspecification.

\begin{figure*}[h]
    \centering

    \begin{minipage}{0.24\textwidth}
        \centering\textbf{\textsc{TwoRoom}}
    \end{minipage}
    \begin{minipage}{0.24\textwidth}
        \centering\textbf{\textsc{OGB}}
    \end{minipage}
    \begin{minipage}{0.24\textwidth}
        \centering\textbf{\textsc{Reacher}}
    \end{minipage}
    \begin{minipage}{0.24\textwidth}
        \centering\textbf{\textsc{PushT}}
    \end{minipage}

    \vspace{1mm}

    \begin{subfigure}[t]{0.24\textwidth}
        \centering
        \includegraphics[width=\linewidth]{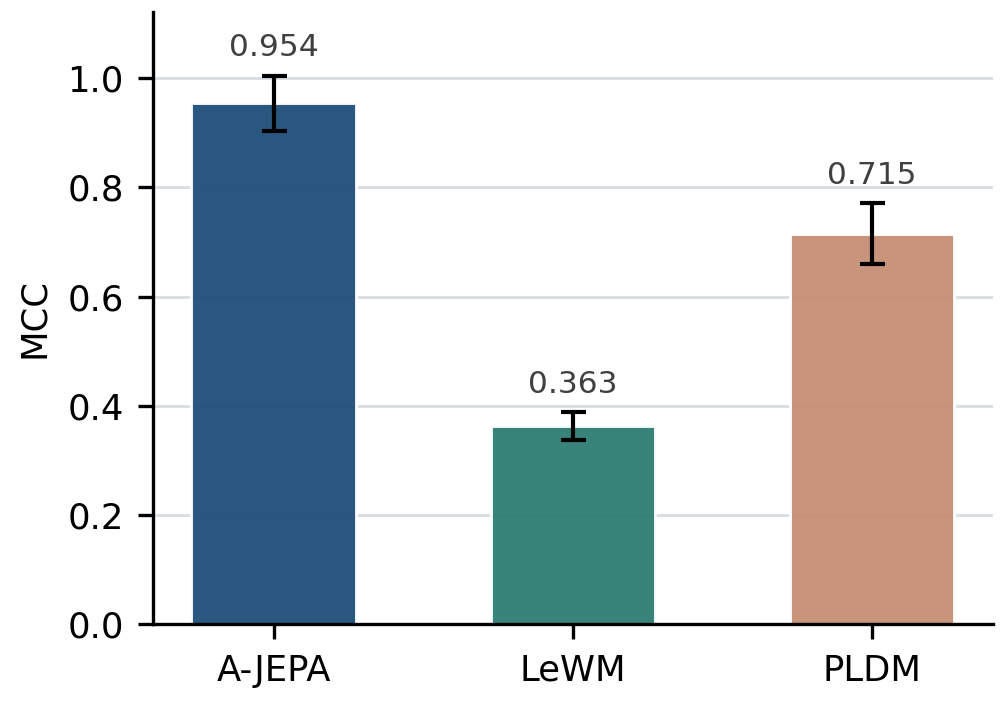}
        \caption{MCC}
    \end{subfigure}
    \hfill
    \begin{subfigure}[t]{0.24\textwidth}
        \centering
        \includegraphics[width=\linewidth]{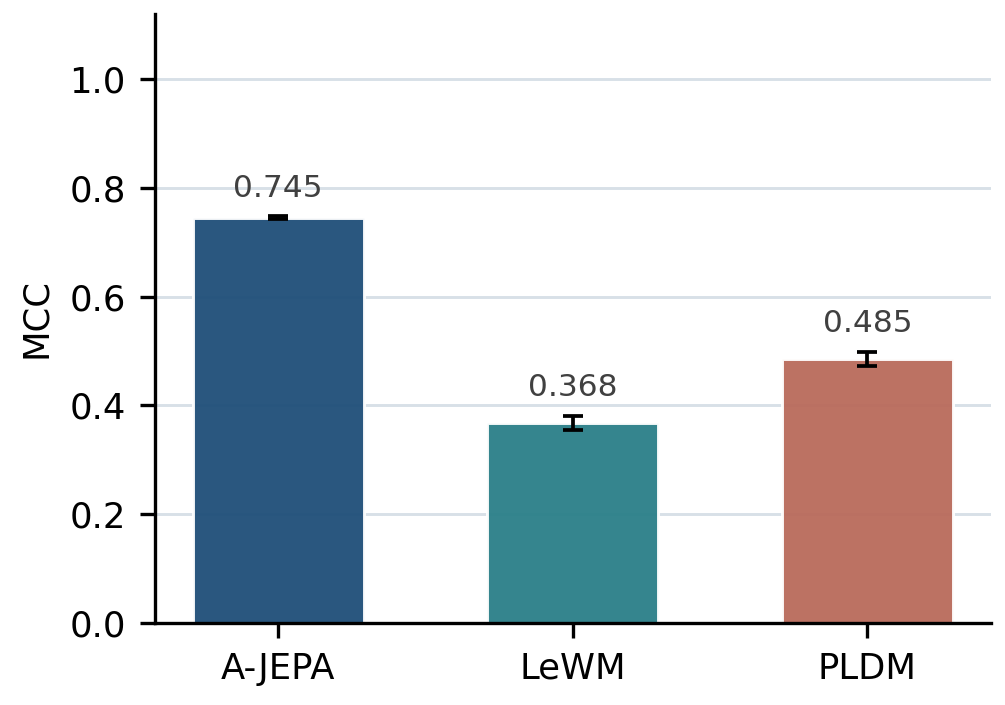}
        \caption{MCC}
    \end{subfigure}
    \hfill
    \begin{subfigure}[t]{0.24\textwidth}
        \centering
        \includegraphics[width=\linewidth]{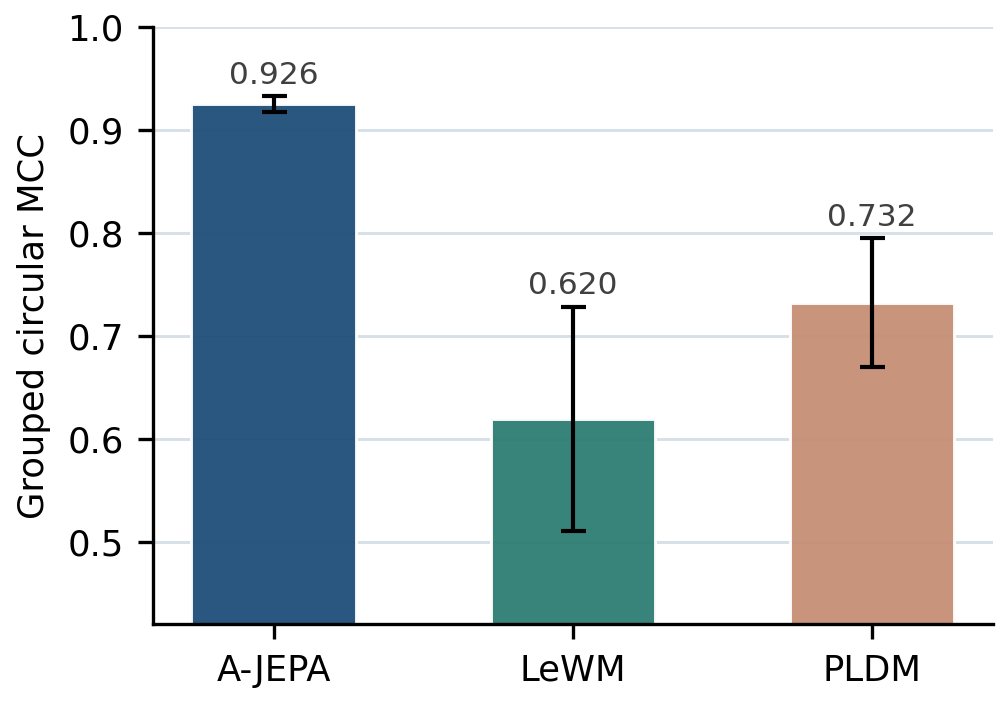}
        \caption{MCC}
    \end{subfigure}
    \hfill
    \begin{subfigure}[t]{0.24\textwidth}
        \centering
        \includegraphics[width=\linewidth]{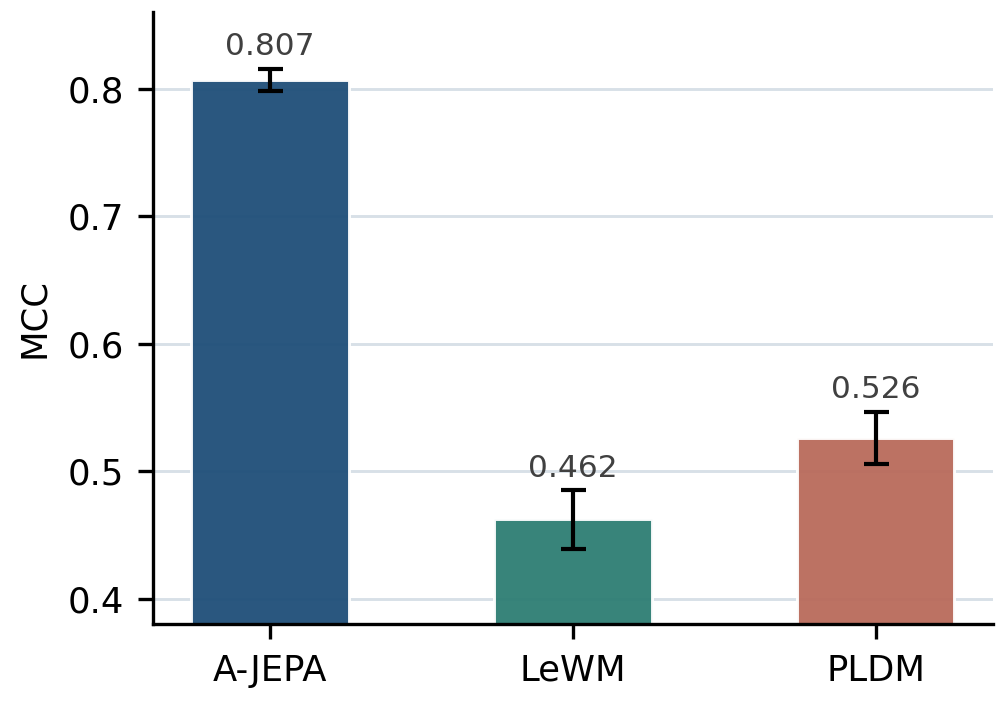}
        \caption{MCC}
    \end{subfigure}

    \vspace{1mm}

    \begin{subfigure}[t]{0.24\textwidth}
        \centering
        \includegraphics[width=\linewidth]{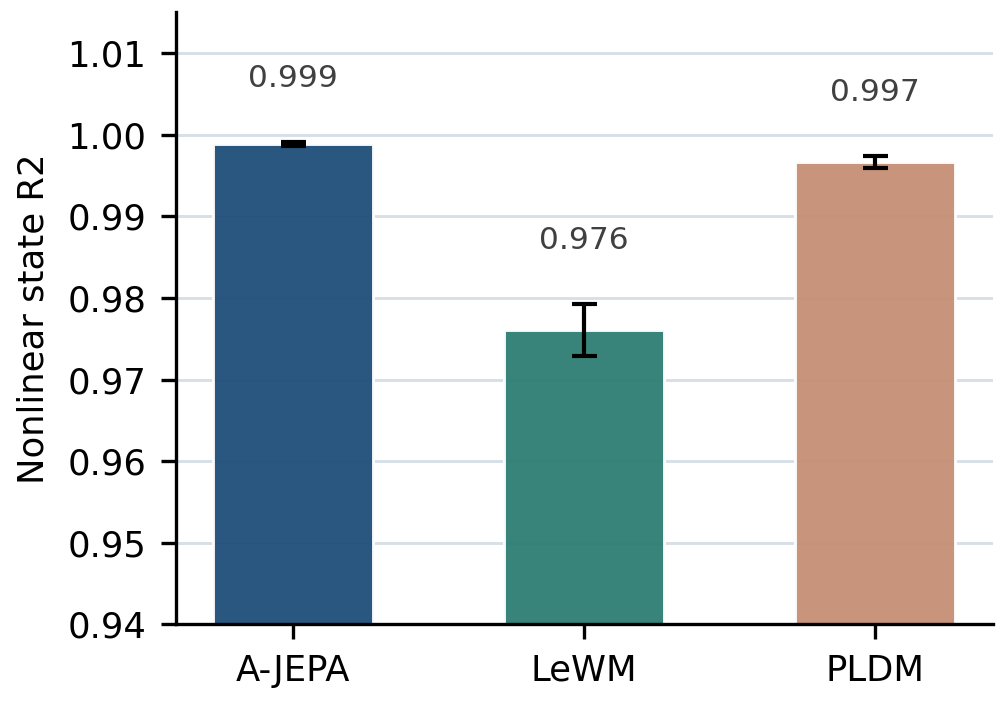}
        \caption{State $R^2$}
    \end{subfigure}
    \hfill
    \begin{subfigure}[t]{0.24\textwidth}
        \centering
        \includegraphics[width=\linewidth]{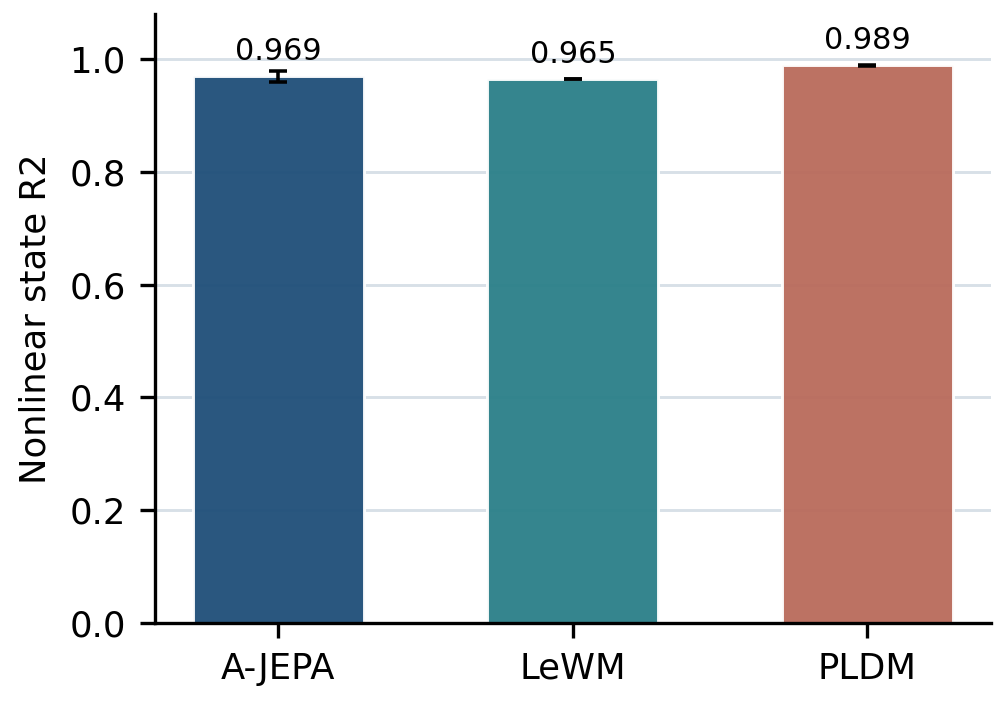}
        \caption{State $R^2$}
    \end{subfigure}
    \hfill
    \begin{subfigure}[t]{0.24\textwidth}
        \centering
        \includegraphics[width=\linewidth]{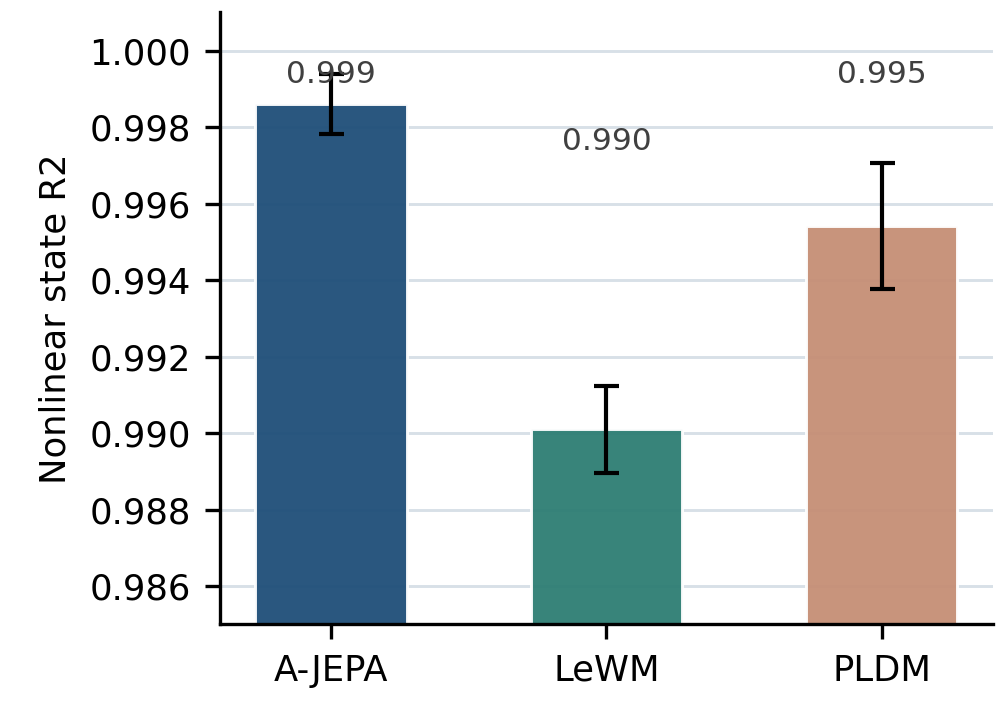}
        \caption{State $R^2$}
    \end{subfigure}
    \hfill
    \begin{subfigure}[t]{0.24\textwidth}
        \centering
        \includegraphics[width=\linewidth]{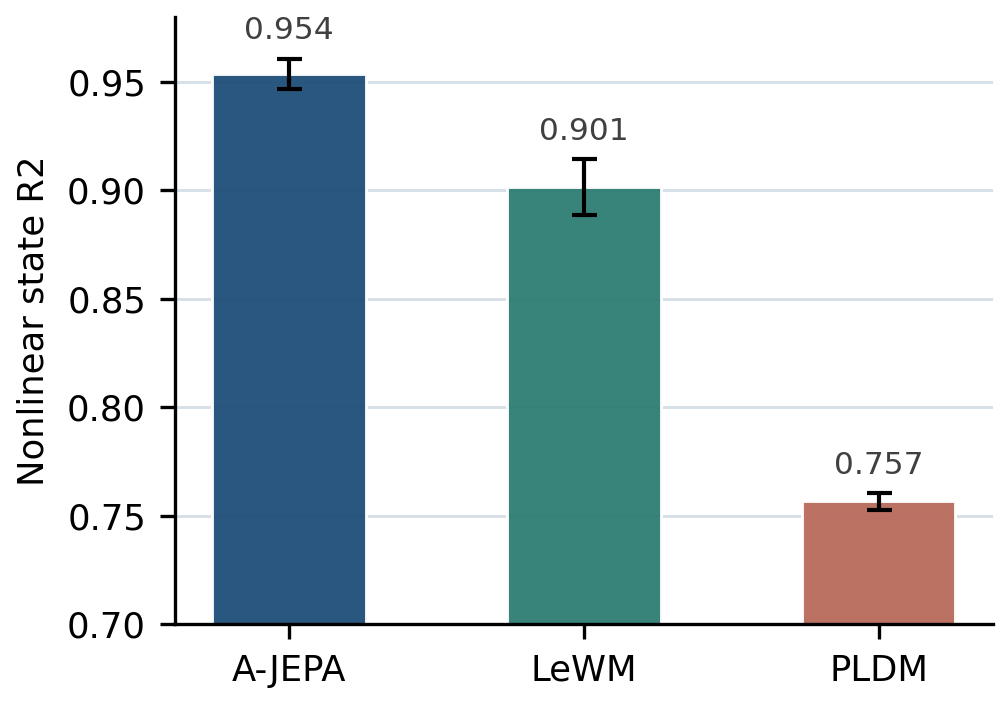}
        \caption{State $R^2$}
    \end{subfigure}
    \vspace{1mm}

    \begin{subfigure}[t]{0.24\textwidth}
        \centering
        \includegraphics[width=\linewidth]{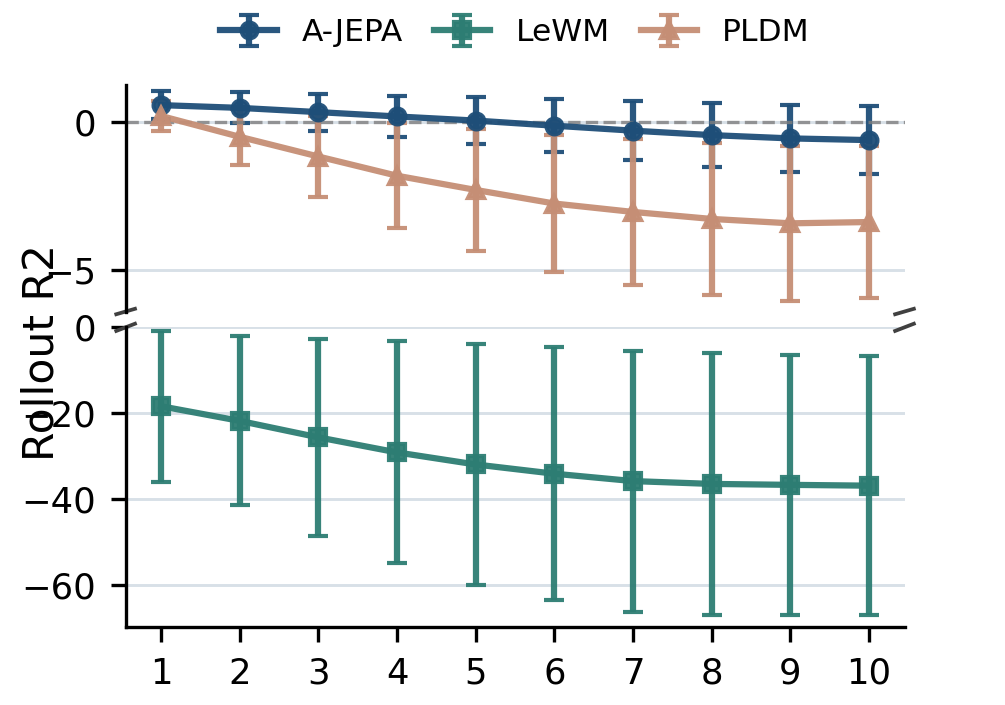}
        \caption{Rollout $R^2$}
    \end{subfigure}
    \hfill
    \begin{subfigure}[t]{0.24\textwidth}
        \centering
        \includegraphics[width=\linewidth]{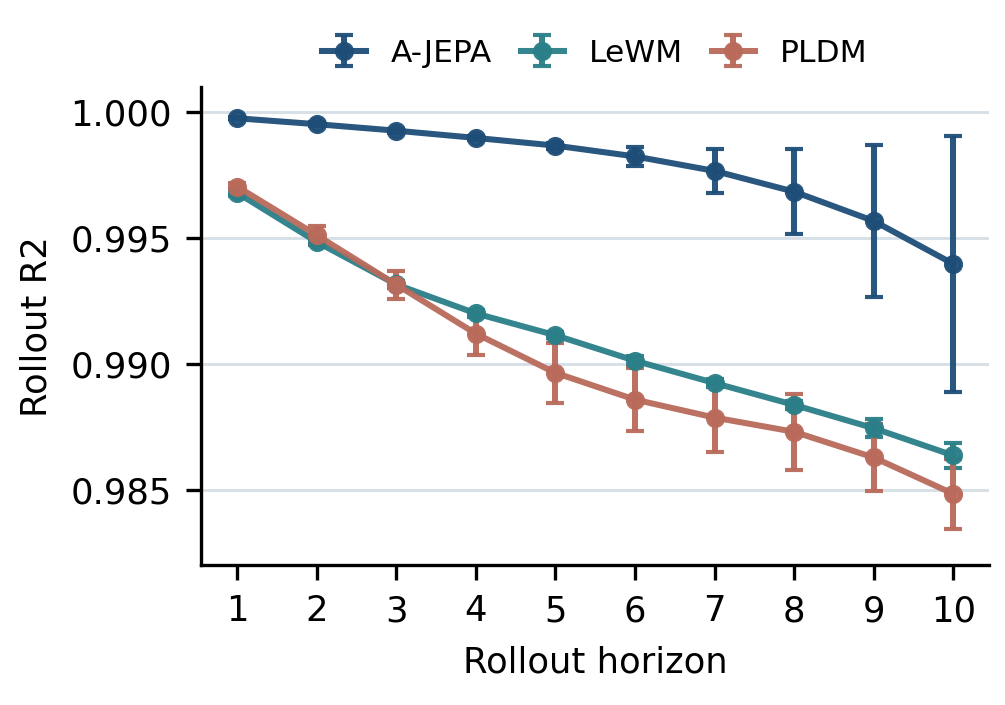}
        \caption{Rollout $R^2$}
    \end{subfigure}
    \hfill
    \begin{subfigure}[t]{0.24\textwidth}
        \centering
        \includegraphics[width=\linewidth]{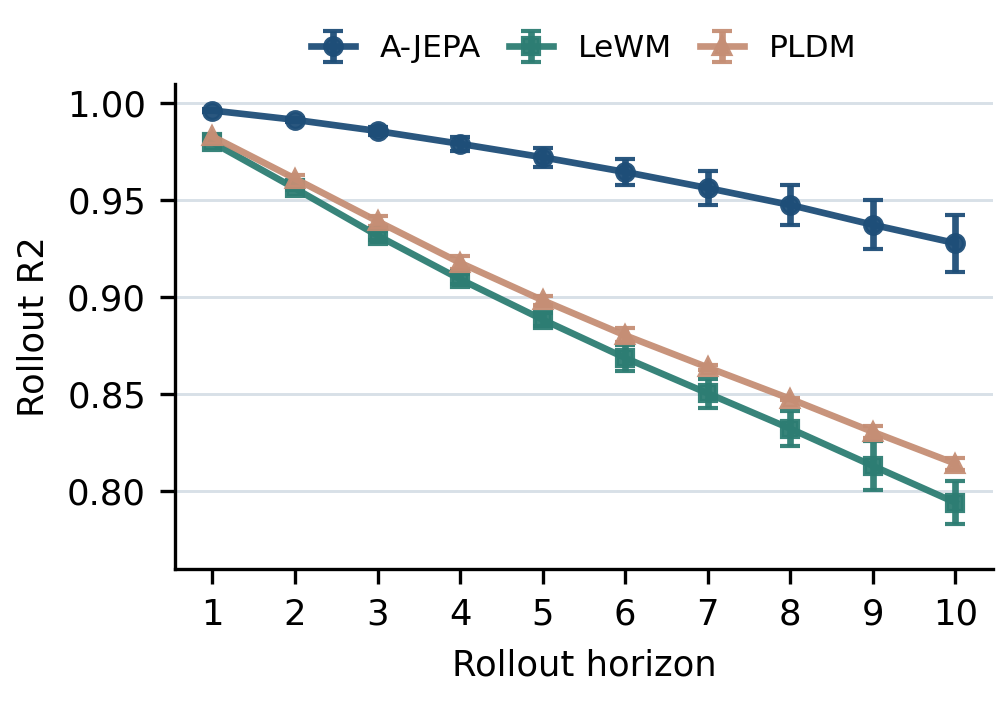}
        \caption{Rollout $R^2$}
    \end{subfigure}
    \hfill
    \begin{subfigure}[t]{0.24\textwidth}
        \centering
        \includegraphics[width=\linewidth]{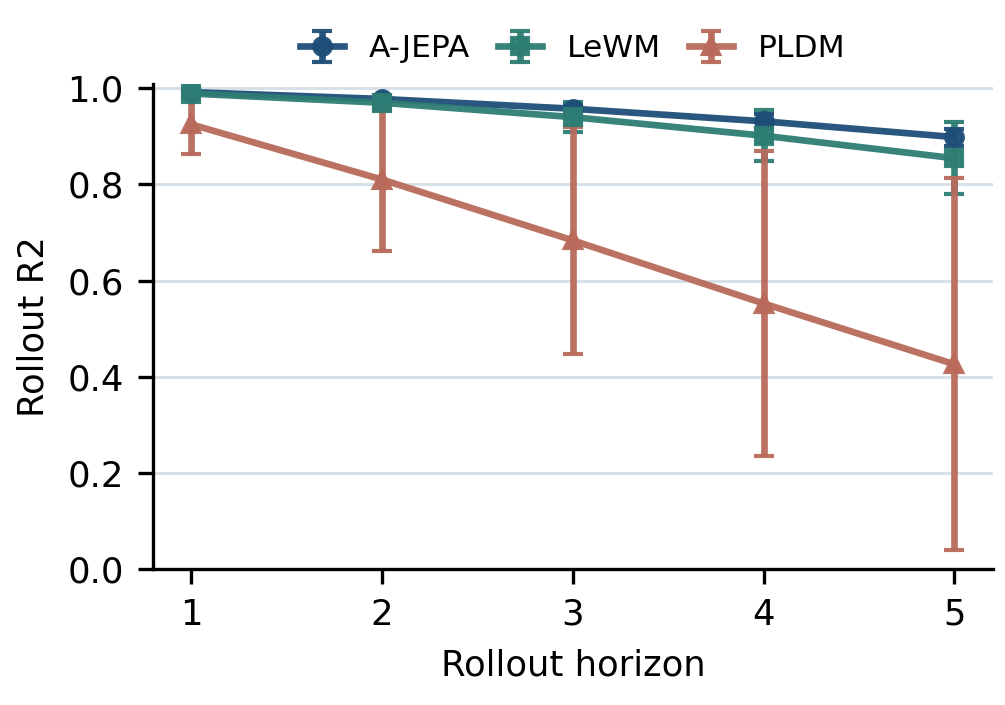}
        \caption{Rollout $R^2$}
    \end{subfigure}

    \caption{Representation and transition recovery across four visual control environments.Columns correspond to \textsc{TwoRoom}, \textsc{OGB}, \textsc{Reacher}, and \textsc{PushT}, while rows report component-wise state recovery (MCC), joint-state information preservation (nonlinear state $R^2$), and multi-step action-conditioned prediction (rollout $R^2$), respectively. A-JEPA consistently improves component-wise recovery while retaining strong state information and predictive dynamics.
}
    \label{fig:visual-control-main}
\end{figure*}
\subsection{Visual Control Environments}
\label{sec:visual-control}

We next examine whether the theoretical and empirical findings above extend beyond the controlled simulation setting to visual environments with more complex action-conditioned dynamics. We consider four environments: \textsc{TwoRoom}~\citep{sobal2025stress}, \textsc{OGB}~\citep{ogbench_park2025}, \textsc{Reacher}~\citep{tassa2018deepmindcontrolsuite}, and \textsc{PushT}~\citep{zhou2025dinowm}, covering increasingly challenging state transitions and interactions. We compare the proposed A-JEPA with two JEPA-style baselines, LeWM~\citep{maes2026leworldmodel} and PLDM~\citep{sobal2026learning}. We evaluate the learned representations from three complementary perspectives. First, we measure component-wise recovery using MCC, which directly evaluates whether individual learned coordinates align with the underlying state factors. Second, we train a nonlinear decoder from the learned representations to the ground-truth states and report nonlinear state $R^2$, providing a proxy for how much joint-state information remains recoverable from the representation irrespective of coordinate-wise alignment. Third, we recursively apply the learned transition model under action sequences and report rollout $R^2$, measuring how well the learned representation supports action-conditioned transition prediction over increasing horizons. See Sec.~\ref{app:visual-benchmark-details}.

\paragraph{Component-Wise State Recovery.}
As shown in the first row of Fig.~\ref{fig:visual-control-main}, A-JEPA
consistently achieves higher MCC than the JEPA-style baselines
across all four environments. The improvement is observed in both simpler
environments such as \textsc{TwoRoom} and \textsc{OGB}, and more complex
continuous-control environments including \textsc{Reacher} and \textsc{PushT}.
These results show that the component-wise recovery advantage persists in more complex visual environments.

\paragraph{Information Preservation Is Not Component-Wise Recovery.}
We next ask whether the lower MCC of the baselines merely reflects information
loss. The middle row of Fig.~\ref{fig:visual-control-main} shows that all
methods retain substantial information about the underlying joint state when
evaluated with a nonlinear decoder. Importantly, high state information does
not necessarily imply component-wise recovery. For example, on
\textsc{OGB}, PLDM achieves nonlinear state $R^2=0.989$ but an MCC of only
$0.485$. This distinction is consistent with our theoretical analysis:
information preservation alone does not rule out mixing between latent causal
factors.

\paragraph{Action-Conditioned Transition Prediction.}
The bottom row of Fig.~\ref{fig:visual-control-main} evaluates recursive
action-conditioned prediction over increasing rollout horizons. A-JEPA
generally exhibits more stable long-horizon predictions, with particularly
clear improvements on \textsc{Reacher} and \textsc{PushT}. Taken together,
these results show that the improved component-wise recovery of A-JEPA is
achieved while retaining substantial state information and strong predictive
dynamics. We emphasize that this establishes an empirical association rather
than implying that component-wise identifiability alone guarantees improved
rollout performance.

\subsection{Robot Mechanism Transfer.}
A key motivation for learning causal transition mechanisms is their potential to transfer beyond the environments in which they are learned. If the learned representation primarily captures environment-specific visual correlations, its transition dynamics are likely to degrade under environmental changes. In contrast, representations that capture causal action-conditioned mechanisms should support more stable predictive dynamics across such shifts. In our final experiment, we therefore examine whether the learned action-conditioned dynamics transfer beyond the training environments.

\begin{wrapfigure}{r}{0.48\textwidth}
    \centering
    \vspace{-0.8em}
    \includegraphics[width=\linewidth]{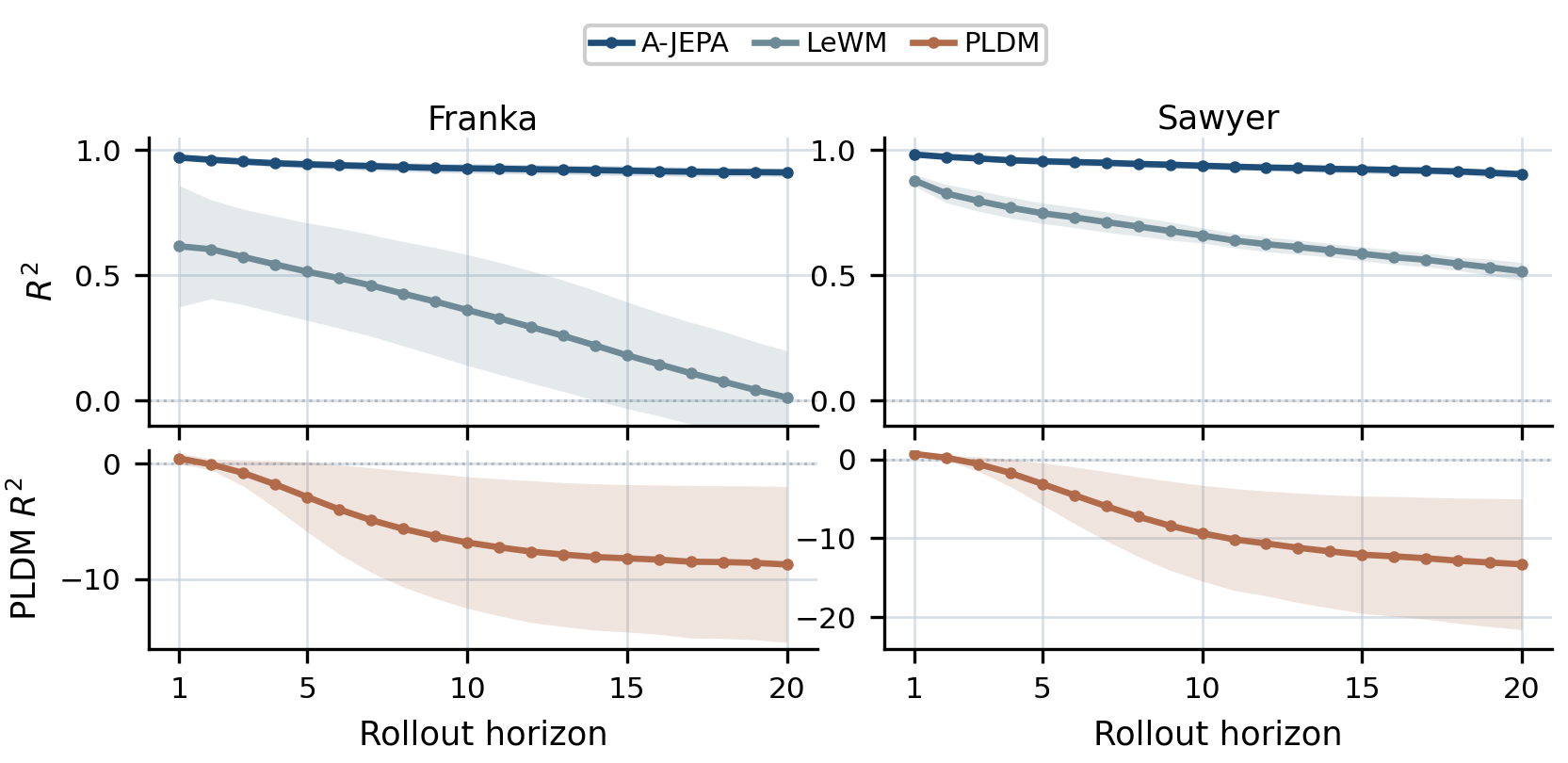}
    \caption{Held-out robot transfer on RoboNet. A-JEPA maintains substantially more stable latent rollout predictions on unseen Franka and Sawyer robots. See Sec.~\ref{app:visual-transfer-details} for details.
    }
    \label{fig:robonet-transfer}
    \vspace{-1.0em}
\end{wrapfigure}
We construct a held-out-robot experiment on the RoboNet dataset~\citep{dasari2019robonet}, which contains trajectories collected from multiple robot platforms. We train the models on Baxter, Kuka, and WidowX and evaluate them without adaptation on held-out Franka and Sawyer robots. This split introduces substantial changes in robot morphology, visual appearance, and embodiment. At test time, all model parameters are frozen. Given the encoded state of a held-out robot and its subsequent action sequence, we recursively roll the learned transition model forward for up to $20$ steps. At each horizon $h$, we compare the predicted latent state $\hat{\mathbf z}_{t+h}$ with the encoding of the corresponding future observation, $\mathbf z_{t+h}=\boldsymbol{\mathrm h}(\mathbf x_{t+h})$, using rollout $R^2$, where larger \(R^2\) values denote more accurate latent rollout prediction. As shown in Fig.~\ref{fig:robonet-transfer}, A-JEPA maintains stable latent
rollout predictions on both held-out robots, with rollout $R^2$ remaining
above $0.90$ at horizon $20$ for both Franka and Sawyer. In contrast,
matched-capacity LeWM and PLDM baselines exhibit substantially stronger
degradation as the rollout horizon increases. Since all methods use the same
encoder architecture and latent dimension and are evaluated without adaptation,
the results provide evidence that A-JEPA learns action-conditioned latent
dynamics with stronger transfer across unseen robot platforms.

\section{Conclusion}
We studied when and how JEPA can recover latent causal states from action-conditioned observations. We developed an information-theoretic objective and established component-wise identifiability under sufficient action-induced variation in the transition mechanisms, which motivates the practical A-JEPA formulation. Experiments on synthetic, visual, and held-out robot settings support the theory and show improved state recovery, robustness, and transfer of the learned dynamics.

\clearpage
\newpage
\section*{AI Use Statement}
Generative AI tools were used to assist with language editing. All scientific ideas, theoretical results, experimental design, analysis, and conclusions were developed and verified by the authors. The authors take full responsibility for the content of this paper.

\section*{Ethics Statement}
This work studies representation learning using publicly available models and datasets. It does not involve human subjects or private personal data.

\section*{Reproducibility Statement}
We provide the assumptions and complete proofs of the theoretical results in the appendix. Detailed experimental protocols, dataset processing procedures, model configurations, and evaluation metrics are also provided in the appendix to facilitate reproduction of our results.



\bibliography{reference}
\bibliographystyle{iclr2027_conference}


\newpage
\appendix
\onecolumn
\part{Appendix} 
\parttoc         %

\newpage
\section{Related Work}
\label{app: rw}
\paragraph{Development of JEPA.}
The JEPA framework was originally inspired by the human brain's perception system~\citep{lecun2022path} to comprehend a dynamic environment in the representation space, which has become a core component of autonomous intelligent systems. Later, JEPA was employed as a self-supervised learning paradigm, such as shown in I-JEPA~\citep{assran2023self}, MC-JEPA~\citep{bardes2023mc}, and LeJEPA~\citep{balestriero2025lejepa}. Empowered by the strong self-supervision capability, JEPA brings significant advancements to various modalities such as video, audio, 3D vision, point cloud, etc.~\citep{bardes2024revisiting, fei2023jepa, hu20243d, saito2025point}, as well as multimodal learning~\citep{lei2024m3, chen2026vl, cornelissen2026mumo}. Further, through large-scale pretraining, it is shown that JEPA is capable of fulfilling a world model~\citep{ha2018world} to enable understanding, predicting, and planning in realistic environments~\citep{assran2025v, mur2026v, zhang2026thinkjepa, maes2026leworldmodel}. Recently, the significance of JEPA has extended beyond AI, touching on realistic and fundamental applications such as medical imaging, spatial transcriptomics, antibody engineering, healthcare, and particle physics~\citep{zhou2026beyond, rabinowitz2026kukulu, zhou2026cellos, birk2026multi, shihabi2026hi, kim2026lightweight, letellier2026jetparticle}. The proposed A-JEPA serves as a backbone with provable identifiability guarantees and might be readily adapted to a broad range of JEPA variants. Under the mild conditions specified in Assumptions~\ref{assump:regularity} and~\ref{assump:action-variability}, A-JEPA remains applicable to practical settings involving model scaling, dataset scaling, and multimodal learning. Overall, this work advances JEPA research by combining broadly generalizable empirical improvements with rigorous theoretical guarantees, while providing valuable insights into action-based JEPA training.

Recent theoretical analysis of JEPA provides latent-variable recovery guarantees up to an orthogonal transformation~\citep{klindt2026does}, under the assumption of Gaussian latent dynamics. Related work on controlled world models considers Gaussian latent states and establishes joint state and transition identifiability under spectral separation and non-degenerate conditional action excitation, again up to an orthogonal transformation~\citep{zhang2026identifiability}. Another line of work removes the Gaussian latent-distribution assumption by introducing a physics-grounded symbolic world model, under the assumption that a predefined symbolic atom library contains a causal basis for the underlying world generator~\citep{dobrin2026identifiability}. In contrast, our analysis assumes neither a Gaussian form for the latent transition mechanisms nor a known symbolic basis, and establishes component-wise identifiability for general nonparametric causal transitions.

\paragraph{Nonlinear ICA.}
This work is related to nonlinear ICA
~\citep{hyvarinen2016unsupervised,hyvarinen2019nonlinear,hyvarinen2017nonlinear,khemakhem2020variational}, which studies the identifiability of latent components from nonlinear observations, typically under assumptions of mutual independence among latents. In particular, two key assumptions in our analysis, Assumptions~\ref{assump:regularity} and~\ref{assump:action-variability}, are closely related to the invertible observation mapping and variability conditions commonly used in nonlinear ICA identifiability theory. Our setting differs in that the latent variables are causal states whose transition mechanisms may depend on one another through a temporal causal graph, rather than mutually independent latent components, which introduces additional structure and also solution ambiguity.

\paragraph{Causal Representation Learning} This work is also related to causal representation learning (CRL)~\citep{scholkopf2021toward}, which aims to recover latent causal variables from observations. A broad line of work has established identifiability results by exploiting temporal dependence ~\citep{yao2021learning,song2023temporally,chen2024caring,li2025identification,lippe2022citris,lachapelle2022disentanglement}. Different from most of these works, which typically establish identifiability through likelihood matching in observational space, we study identifiability within the JEPA framework, where representations are identified through action-conditioned prediction in latent
space. Further, another line of work exploits interventions or multiple environments for identifiability ~\citep{brehmer2022weakly,von2024nonparametric,ahuja2023interventional,seigal2022linear,buchholz2023learning,varici2023score,shen2022weakly,liuidentifiable,JMLR:v26:22-0136}. Many of these works focus on instantaneous causal structures. while we consider temporal causal states and exploit action-induced variation in these mechanisms for identifiability.

\paragraph{Model-Based Reinforcement Learning}
Action-conditioned latent prediction has been extensively studied in model-based reinforcement learning. PlaNet and Dreamer learn latent dynamics by matching an action-conditioned predictive prior to an observation-conditioned posterior, enabling planning or policy learning through imagined latent trajectories \citep{hafner2019learning,hafner2019dream,hafner2025mastering}. Other world-model approaches also learn action-conditioned latent dynamics without requiring full observation reconstruction: MuZero learns latent dynamics specialized for predicting quantities relevant to planning \citep{schrittwieser2020mastering}, while TPC, EfficientZero, and TD-MPC/TD-MPC2 employ temporal prediction, latent consistency, or task-oriented latent dynamics for planning and control \citep{nguyen2021temporal,ye2021mastering,hansen2022temporal,hansen2024td}. Closely related self-predictive approaches such as SPR explicitly predict future latent representations through learned transition models \citep{schwarzer2020data}, while reconstruction-free representation-learning methods based on bisimulation seek task-relevant latent abstractions without pixel decoding \citep{zhang2020learning}. More recently, DreamerPro and R2-Dreamer replace observation reconstruction with self-supervised representation regularization \citep{deng2022dreamerpro,morihira2026r2}. These works demonstrate a broader pattern of learning action-conditioned latent dynamics together with mechanisms that preserve task-relevant or non-degenerate representations. However, they primarily target predictive and control performance rather than establishing recovery of the underlying causal state up to component-wise transformations. Our results address this identifiability question for this broader family of predictive representation objectives.

\section{Limitations}
\label{app:limit}

One of the main contributions of this work is establishing a component-wise identifiability result under the JEPA framework. As in existing identifiability analyses, our result relies on several assumptions on the underlying generative process. In particular, we assume that the observation mapping is smooth and invertible, such that information about the latent state is preserved through the observation process. Closely related assumptions are standard in nonlinear ICA and causal representation learning~\citep{hyvarinen2016unsupervised,hyvarinen2019nonlinear,khemakhem2020variational,von2024nonparametric,JMLR:v26:22-0136}. We further require sufficient action-induced variation in the latent transition mechanisms, formalized through a full-rank variability condition. This type of variability condition also has a well-established precedent: it originates from identifiability analyses for nonlinear ICA ~\citep{hyvarinen2016unsupervised,hyvarinen2019nonlinear, khemakhem2020variational} and has subsequently been adapted, in different forms, to causal representation learning ~\citep{zhang2024causal,JMLR:v26:22-0136,yao2021learning}.

As in other identifiability analyses of latent-variable models, directly verifying these assumptions in real-world settings is inherently challenging, because the underlying latent generative process is unobserved. Nevertheless, our controlled experiments show strong recovery under the matched setting and robustness when observation invertibility, action-induced mechanism variation, and transition-noise specification are moderately perturbed. Developing weaker identifiability conditions and tighter connections between the general objective and practical estimators remains an important direction for future work.

\newpage
\section{Proof of Theorem~\ref{thm:optimal-solution}}
\label{app:proof-optimal-solution}

\begin{proof}
For brevity, let $\mathbf c_t := (\mathbf z_{t-1},\mathbf a_{t-1})$. For a given encoder $\boldsymbol{\mathrm h}$, $Q_{\boldsymbol{\mathrm h}}(\mathbf z_t\mid\mathbf c_t)$ denotes the conditional transition distribution induced in the representation space. Then, the conditional negative log-likelihood in the proposed objective Eq.~\eqref{eq:information-jepa-objective} can be decomposed as
\begin{align}
-\mathbb E \left[ \log p_{\boldsymbol{\phi}}(\mathbf z_t\mid\mathbf c_t) \right]=& -\mathbb E_{\mathbf c_t}
\mathbb E_{\mathbf z_t\sim Q_{\boldsymbol{\mathrm h}}(\cdot\mid\mathbf c_t)} \left[ \log Q_{\boldsymbol{\mathrm h}}
(\mathbf z_t\mid\mathbf c_t) \right]+ \mathbb E_{\mathbf c_t} \mathbb E_{\mathbf z_t\sim Q_{\boldsymbol{\mathrm h}}(\cdot\mid\mathbf c_t)} \left[ \log \frac{ Q_{\boldsymbol{\mathrm h}} (\mathbf z_t\mid\mathbf c_t) }{
p_{\boldsymbol{\phi}} (\mathbf z_t\mid\mathbf c_t) } \right] \nonumber\\
=& H(\mathbf z_t\mid\mathbf c_t) + \mathbb E_{\mathbf c_t} \left[ D_{\mathrm{KL}} \left( Q_{\boldsymbol{\mathrm h}}(\cdot\mid\mathbf c_t) \Vert p_{\boldsymbol{\phi}}(\cdot\mid\mathbf c_t) \right) \right].
\label{eq:proof-cross-entropy}
\end{align}
Substituting Eq.~\eqref{eq:proof-cross-entropy} into the proposed objective Eq.~\eqref{eq:information-jepa-objective} and using $H(\mathbf z_t\mid\mathbf c_t) = H(\mathbf z_t) - I(\mathbf z_t;\mathbf c_t)$, we obtain
\begin{align}
\mathcal L_{\lambda}(\boldsymbol{\mathrm h},\boldsymbol{\phi}) = - I(\mathbf z_t;\mathbf z_{t-1},\mathbf a_{t-1})
+ \mathbb E_{\mathbf c_t} \left[ D_{\mathrm{KL}} \left( Q_{\boldsymbol{\mathrm h}}(\cdot\mid\mathbf c_t)
\,\Vert\, p_{\boldsymbol{\phi}}(\cdot\mid\mathbf c_t) \right) \right]+ (\lambda-1)\big(-H(\mathbf z_t)\big).
\label{eq:proof-objective-decomposition}
\end{align}

Since $\mathbf z_t=\boldsymbol{\mathrm h}(\boldsymbol{\mathrm g}(\mathbf s_t))$ is a deterministic function of $\mathbf s_t$, and thus $(\mathbf z_{t-1},\mathbf a_{t-1})$ is a deterministic function of $(\mathbf s_{t-1},\mathbf a_{t-1})$, applying the data processing inequality to the two arguments gives
\begin{align}
I(\mathbf z_t;\mathbf z_{t-1},\mathbf a_{t-1})
\le I(\mathbf s_t;\mathbf z_{t-1},\mathbf a_{t-1}) \le
I(\mathbf s_t;\mathbf s_{t-1},\mathbf a_{t-1}).
\label{eq:proof-dpi}
\end{align}
Moreover, because
$\boldsymbol{\mathrm h}:\mathcal X\rightarrow(0,1)^d$,
the differential entropy of $\mathbf z_t$ satisfies
\begin{equation}
H(\mathbf z_t)
\leq
\log \operatorname{Vol}((0,1)^d)
=
0.
\label{eq:proof-entropy-bound}
\end{equation}
Together with the non-negativity of the KL divergence and the data-processing
inequality in Eq.~\eqref{eq:proof-dpi}, Eq.~\eqref{eq:proof-objective-decomposition}
gives
\begin{align}
\mathcal L_{\lambda}(\boldsymbol{\mathrm h},\boldsymbol{\phi}) + I(\mathbf s_t;\mathbf s_{t-1},\mathbf a_{t-1})
=&\; \underbrace{I(\mathbf s_t;\mathbf s_{t-1},\mathbf a_{t-1})-I(\mathbf z_t;\mathbf z_{t-1},\mathbf a_{t-1})}_{\geq 0}
\nonumber\\
&+\underbrace{\mathbb E_{\mathbf c_t}\left[D_{\mathrm{KL}}\left(Q_{\boldsymbol{\mathrm h}}(\cdot\mid\mathbf c_t)\Vert
p_{\boldsymbol{\phi}}(\cdot\mid\mathbf c_t)\right)\right]}_{\geq 0}\nonumber\\
&+\underbrace{(\lambda-1)\big(-H(\mathbf z_t)\big)}_{\geq 0},
\label{eq:proof-nonnegative-terms}
\end{align}
where the last term is non-negative when $\lambda\geq1$. As a result, we have:
\begin{equation}
\mathcal L_{\lambda}(\boldsymbol{\mathrm h},\boldsymbol{\phi}) \geq - I(\mathbf s_t;\mathbf s_{t-1},\mathbf a_{t-1}).
\end{equation}

When this lower bound is attainable, any global minimizer attaining the bound must make all three non-negative terms in Eq.~\eqref{eq:proof-nonnegative-terms} equal to zero. That is,
\begin{equation}
I(\mathbf z_t;\mathbf z_{t-1},\mathbf a_{t-1}) = I(\mathbf s_t;\mathbf s_{t-1},\mathbf a_{t-1}),
\label{app:eq:info}
\end{equation}
and
\begin{equation}
p_{\boldsymbol{\phi}^{\star}} (\mathbf z_t\mid\mathbf z_{t-1},\mathbf a_{t-1}) = Q_{\boldsymbol{\mathrm h}^{\star}} (\mathbf z_t\mid\mathbf z_{t-1},\mathbf a_{t-1}) \quad\mathrm{a.e.}
\end{equation}
This completes the proof. As an additional consequence, when $\lambda>1$, Eq.~\eqref{eq:proof-nonnegative-terms} also implies $H(\mathbf z_t)=0$. Since the maximum differential entropy over
$(0,1)^d$ is $0$, attained uniquely by the uniform distribution, we obtain $\mathbf z_t\sim\mathrm{Unif}((0,1)^d)$.
\end{proof}

\newpage
\section{Proof of Theorem~\ref{thm:componentwise-identifiability}}
\label{app:componentwise-identifiability}
\paragraph{Proof sketch.}
The proof proceeds in three steps.

\textbf{Step I: Local invertibility.}
Using the information-preservation result in
Theorem~\ref{thm:optimal-solution} together with sufficient action
variability, we show that the Jacobian of
$\boldsymbol r=\boldsymbol h\circ\boldsymbol g$ has full rank everywhere, which implies that $\boldsymbol r$ is locally invertible.

\textbf{Step II: Ruling out cross-component mixing.}
Using the local inverse of $\boldsymbol r$, together with the
sufficient action variability, i.e., Assumption~\ref{assump:action-variability}, we show that each row of the Jacobian of
the local inverse contains at most one nonzero entry.

\textbf{Step III: Component-wise identifiability.} Combining Steps~I and~II, the Jacobian of the local inverse must be a generalized permutation matrix. Under regularity of the latent domain, the permutation pattern is fixed throughout the domain, yielding component-wise invertible transformations of the latent causal states.
\begin{proof}
\noindent\textbf{Step I: Local invertibility of $\boldsymbol r$.} Let $\boldsymbol r=\boldsymbol h\circ\boldsymbol g$, $\mathbf z_t=\boldsymbol r(\mathbf s_t)$. By Theorem~\ref{thm:optimal-solution}, specifically
Eq.~\eqref{app:eq:info}, every global minimizer satisfies:
\begin{equation}
I(\mathbf z_t;\mathbf z_{t-1},\mathbf a_{t-1}) = I(\mathbf s_t;\mathbf s_{t-1},\mathbf a_{t-1}).
\label{eq:proof-mi-equality}
\end{equation}
Since $\mathbf z_t$ is a deterministic function of $\mathbf s_t$, and accordingly, $(\mathbf z_{t-1},\mathbf a_{t-1})$ is a deterministic function of $(\mathbf s_{t-1},\mathbf a_{t-1})$, the data-processing inequality gives
\begin{equation}
I(\mathbf z_t;\mathbf z_{t-1},\mathbf a_{t-1}) \leq I(\mathbf z_t;\mathbf s_{t-1},\mathbf a_{t-1}) \leq I(\mathbf s_t;\mathbf s_{t-1},\mathbf a_{t-1}).
\end{equation}
Together with Eq.~\eqref{eq:proof-mi-equality}, both inequalities must be equalities, as follows:
\begin{equation}
I(\mathbf z_t;\mathbf z_{t-1},\mathbf a_{t-1}) = I(\mathbf z_t;\mathbf s_{t-1},\mathbf a_{t-1}) = I(\mathbf s_t;\mathbf s_{t-1},\mathbf a_{t-1}).
\label{app:eq:infoequ}
\end{equation}

By the chain rule for mutual information, i.e.,
$I(X;Z\mid Y)=I(X,Y;Z)-I(Y;Z)$, together with
$\mathbf z_t=\boldsymbol r(\mathbf s_t)$, we have
\begin{align}
I(\mathbf s_t;\mathbf s_{t-1},\mathbf a_{t-1}\mid\mathbf z_t)
&=I(\mathbf s_t,\mathbf z_t;\mathbf s_{t-1},\mathbf a_{t-1})-I(\mathbf z_t;\mathbf s_{t-1},\mathbf a_{t-1}) \label{eq:proof-sufficiency-1:1} \\
&=\underbrace{I(\mathbf s_t;\mathbf s_{t-1},\mathbf a_{t-1})-I(\mathbf z_t;\mathbf s_{t-1},\mathbf a_{t-1})}_{=\,0\ \text{by Eq.~\eqref{app:eq:infoequ}}} \label{eq:proof-sufficiency-1:2} \\
&=0.
\label{eq:proof-sufficiency-1}
\end{align}
Here, $I(\mathbf s_t,\mathbf z_t; \mathbf s_{t-1},\mathbf a_{t-1})$ in Eq.~\eqref{eq:proof-sufficiency-1:1} is equal to $ I(\mathbf s_t; \mathbf s_{t-1},\mathbf a_{t-1})$ in Eq.~\eqref{eq:proof-sufficiency-1:2}, because $\mathbf z_t=\boldsymbol r(\mathbf s_t)$ is a deterministic function of $\mathbf s_t$. Similarly, since $\mathbf z_{t-1}=\boldsymbol r(\mathbf s_{t-1})$, we have
\begin{align}
I(\mathbf z_t;\mathbf s_{t-1}\mid
\mathbf z_{t-1},\mathbf a_{t-1})
&=I(\mathbf z_t;\mathbf s_{t-1},\mathbf z_{t-1},\mathbf a_{t-1})
-I(\mathbf z_t;\mathbf z_{t-1},\mathbf a_{t-1})\label{eq:proof-sufficiency-2:1}\\
&=\underbrace{I(\mathbf z_t;\mathbf s_{t-1},\mathbf a_{t-1})-I(\mathbf z_t; \mathbf z_{t-1},\mathbf a_{t-1})}_{=\,0\ \text{by Eq.~\eqref{app:eq:infoequ}}}\label{eq:proof-sufficiency-2:2}\\
&=0.
\label{eq:proof-sufficiency-2}
\end{align}
Similarly, $I(\mathbf z_t; \mathbf s_{t-1},\mathbf z_{t-1},\mathbf a_{t-1}) $ in Eq.~\eqref{eq:proof-sufficiency-2:1} is equal to $I(\mathbf z_t; \mathbf s_{t-1},\mathbf a_{t-1})$ in Eq.~\eqref{eq:proof-sufficiency-2:2}, because $\mathbf z_{t-1}=\boldsymbol r(\mathbf s_{t-1})$ is a deterministic function of $\mathbf s_{t-1}$.

Eq.~\eqref{eq:proof-sufficiency-1} implies
\begin{align}
p(\mathbf s_{t-1},\mathbf a_{t-1}\mid
\mathbf s_t,\mathbf z_t) = p(\mathbf s_{t-1},\mathbf a_{t-1}\mid\mathbf z_t).
\label{eq:proof-sufficiency-22}
\end{align}
Since $\mathbf z_t=\boldsymbol r(\mathbf s_t)$ is deterministic given
$\mathbf s_t$, the left-hand side in Eq.~\eqref{eq:proof-sufficiency-22} equals
$p(\mathbf s_{t-1},\mathbf a_{t-1}\mid\mathbf s_t)$, yielding
\begin{align}
p(\mathbf s_{t-1},\mathbf a_{t-1}\mid\mathbf s_t)= p(\mathbf s_{t-1},\mathbf a_{t-1}\mid\mathbf z_t).
\end{align}
Applying Bayes' rule to both sides gives
\begin{align}
\frac{p(\mathbf s_t\mid\mathbf s_{t-1},\mathbf a_{t-1})p(\mathbf s_{t-1},\mathbf a_{t-1})}{p(\mathbf s_t)}
=\frac{p(\mathbf z_t\mid\mathbf s_{t-1},\mathbf a_{t-1})p(\mathbf s_{t-1},\mathbf a_{t-1})}{p(\mathbf z_t)}.
\end{align}
Canceling the common term yields
\begin{equation}
\frac{p(\mathbf s_t\mid\mathbf s_{t-1},\mathbf a_{t-1})}{p(\mathbf s_t)}
=\frac{p(\mathbf z_t\mid\mathbf s_{t-1},\mathbf a_{t-1})}{p(\mathbf z_t)}.
\label{eq:proof-ratio-intermediate}
\end{equation}

Similarly, Eq.~\eqref{eq:proof-sufficiency-2} implies $p(\mathbf z_t\mid \mathbf s_{t-1},\mathbf z_{t-1},\mathbf a_{t-1}) = p(\mathbf z_t\mid \mathbf z_{t-1},\mathbf a_{t-1})$. Since $\mathbf z_{t-1}=\boldsymbol r(\mathbf s_{t-1})$ is deterministic given $\mathbf s_{t-1}$, the left-hand side reduces to $p(\mathbf z_t\mid\mathbf s_{t-1},\mathbf a_{t-1})$. Hence,
\begin{equation}
p(\mathbf z_t\mid\mathbf s_{t-1},\mathbf a_{t-1}) = p(\mathbf z_t\mid\mathbf z_{t-1},\mathbf a_{t-1}) = Q_{\boldsymbol h} (\mathbf z_t\mid\mathbf z_{t-1},\mathbf a_{t-1}).
\label{eq:proof-cond-independence-2}
\end{equation}

Substituting Eq.~\eqref{eq:proof-cond-independence-2} into Eq.~\eqref{eq:proof-ratio-intermediate} gives
\begin{equation}
\frac{ p(\mathbf s_t\mid\mathbf s_{t-1},\mathbf a_{t-1}) }{ p(\mathbf s_t) } = \frac{ Q_{\boldsymbol h} (\mathbf z_t\mid\mathbf z_{t-1},\mathbf a_{t-1})}{ p(\mathbf z_t)}.
\label{eq:proof-likelihood-ratio}
\end{equation}

We next exploit the action variability in Assumption~\ref{assump:action-variability}. Fix $(\mathbf s_t,\mathbf s_{t-1})$ and consider the corresponding $2d+1$ actions $\mathbf a_{t-1}^{(0)},\ldots,\mathbf a_{t-1}^{(2d)}$.
For each $\ell=1,\ldots,2d$, taking logarithms of Eq.~\eqref{eq:proof-likelihood-ratio} and subtracting the equation under $\mathbf a_{t-1}^{(0)}$ from that under $\mathbf a_{t-1}^{(\ell)}$ gives
\begin{align}
&\log p(\mathbf s_t\mid\mathbf s_{t-1},\mathbf a_{t-1}^{(\ell)})
-\log p(\mathbf s_t\mid\mathbf s_{t-1},\mathbf a_{t-1}^{(0)}) \nonumber\\
&\qquad= \log Q_{\boldsymbol h} (\mathbf z_t\mid\mathbf z_{t-1},\mathbf a_{t-1}^{(\ell)}) - \log Q_{\boldsymbol h} (\mathbf z_t\mid\mathbf z_{t-1},\mathbf a_{t-1}^{(0)}).
\label{eq:proof-action-difference}
\end{align}

Using the latent transition factorization, $\log p(\mathbf s_t\mid\mathbf s_{t-1},\mathbf a_{t-1})=\sum_{i=1}^d q_i$, where $q_i$ is defined in Assumption~\ref{assump:action-variability}. Differentiating Eq.~\eqref{eq:proof-action-difference} with respect to $\mathbf s_t$ and using $\mathbf z_t=\boldsymbol r(\mathbf s_t)$ gives
\begin{align}
&
\begin{pmatrix}
q_1'(\mathbf a_{t-1}^{(\ell)})
-
q_1'(\mathbf a_{t-1}^{(0)})
\\
\vdots\\
q_d'(\mathbf a_{t-1}^{(\ell)})
-
q_d'(\mathbf a_{t-1}^{(0)})
\end{pmatrix}
=
J_{\boldsymbol r}(\mathbf s_t)^\top
\Big[
\nabla_{\mathbf z_t}
\log Q_{\boldsymbol h}
(\mathbf z_t\mid\mathbf z_{t-1},\mathbf a_{t-1}^{(\ell)})
-
\nabla_{\mathbf z_t}
\log Q_{\boldsymbol h}
(\mathbf z_t\mid\mathbf z_{t-1},\mathbf a_{t-1}^{(0)})
\Big],
\label{eq:proof-gradient-relation}
\end{align}
where $ q_i' = \frac{\partial q_i}{\partial s_{t,i}}$, and $\mathbf z_{t-1}$ is fixed when differentiating with respect to $\mathbf s_t$.

For brevity, with $(\mathbf s_t,\mathbf s_{t-1})$ fixed, write $\mathbf w(\mathbf a_{t-1}) = \mathbf w(\mathbf s_t,\mathbf s_{t-1},\mathbf a_{t-1})$. By Assumption~\ref{assump:action-variability}, the matrix
\begin{align}
\boldsymbol L
=
\left[
\boldsymbol w(\mathbf a_{t-1}^{(1)})
-
\boldsymbol w(\mathbf a_{t-1}^{(0)}),
\ldots,
\boldsymbol w(\mathbf a_{t-1}^{(2d)})
-
\boldsymbol w(\mathbf a_{t-1}^{(0)})
\right]
\end{align}
is invertible, where $ \boldsymbol w = (q_1',\ldots,q_d',q_1'',\ldots,q_d'')^\top$.

Since $\boldsymbol L$ is invertible, the submatrix formed by its first
$d$ rows has rank $d$. Its columns are
\begin{align}
\mathbf v_\ell
=
\begin{pmatrix}
q_1'(\mathbf a_{t-1}^{(\ell)})-q_1'(\mathbf a_{t-1}^{(0)})\\
\vdots\\
q_d'(\mathbf a_{t-1}^{(\ell)})-q_d'(\mathbf a_{t-1}^{(0)})
\end{pmatrix},
\qquad
\ell=1,\ldots,2d,
\end{align}
and therefore span $\mathbb R^d$. By Eq.~\eqref{eq:proof-gradient-relation}, each $\mathbf v_\ell$ can be
written as $\mathbf v_\ell=J_{\boldsymbol r}(\mathbf s_t)^\top \mathbf u_\ell$ for some $\mathbf u_\ell\in\mathbb R^d$. Hence, stacking these vectors column-wise gives
\begin{align}
\boldsymbol V
=
J_{\boldsymbol r}(\mathbf s_t)^\top \boldsymbol U,
\end{align}
with $\operatorname{rank}(\boldsymbol V)=d$. Therefore,
\begin{align}
d
=
\operatorname{rank}(\boldsymbol V)
\leq
\operatorname{rank}(J_{\boldsymbol r}(\mathbf s_t)^\top)
\leq d,
\end{align}
which implies
\begin{equation}
\operatorname{rank}J_{\boldsymbol r}(\mathbf s_t)=d.
\label{eq:proof-full-rank}
\end{equation}
Since the above argument applies throughout the latent support, $J_{\boldsymbol r}(\mathbf s_t)$ has full rank everywhere. By the inverse function theorem, $\boldsymbol r$ is locally invertible around every point in the latent support.

\noindent\textbf{Step II: Ruling out cross-component mixing.}
Consider an arbitrary neighborhood on which a local inverse
$\boldsymbol T$ of $\boldsymbol r$ is defined. Using
$\mathbf s_t=\boldsymbol T(\mathbf z_t)$,
Eq.~\eqref{eq:proof-action-difference} becomes
\begin{align}
\sum_{i=1}^d
\left[
q_i(\mathbf a_{t-1}^{(\ell)})
-
q_i(\mathbf a_{t-1}^{(0)})
\right]
=
\log Q_{\boldsymbol h}
(\mathbf z_t\mid\mathbf z_{t-1},\mathbf a_{t-1}^{(\ell)})
\quad-
\log Q_{\boldsymbol h}
(\mathbf z_t\mid\mathbf z_{t-1},\mathbf a_{t-1}^{(0)}),
\label{eq:proof-local-action-difference}
\end{align}
where the left-hand side is evaluated at
$s_{t,i}=T_i(\mathbf z_t)$.

At the lower-bound-attaining global optimum,
Theorem~\ref{thm:optimal-solution} gives
$Q_{\boldsymbol h}=p_{\boldsymbol\phi}$ a.e.
Moreover, the transition model factorizes across representation components:
\begin{equation}
p_{\boldsymbol\phi}
(\mathbf z_t\mid\mathbf z_{t-1},\mathbf a_{t-1})
=
\prod_{j=1}^d
p_{\boldsymbol\phi,j}
(z_{t,j}\mid\mathbf z_{t-1},\mathbf a_{t-1}).
\label{eq:proof-factorized-transition}
\end{equation}
Hence, for any $j\neq k$,
\begin{equation}
\frac{\partial^2}
{\partial z_{t,j}\partial z_{t,k}}
\log Q_{\boldsymbol h}
(\mathbf z_t\mid\mathbf z_{t-1},\mathbf a_{t-1})
=
0.
\label{eq:proof-zero-mixed}
\end{equation}

Taking the mixed derivative of
Eq.~\eqref{eq:proof-local-action-difference} with respect to
$z_{t,j}$ and $z_{t,k}$ gives
\begin{align}
\sum_{i=1}^d
\Bigg[
&
\left(
q_i'(\mathbf a_{t-1}^{(\ell)})
-
q_i'(\mathbf a_{t-1}^{(0)})
\right)
\frac{\partial^2T_i}
{\partial z_{t,j}\partial z_{t,k}}
+
\left(
q_i''(\mathbf a_{t-1}^{(\ell)})
-
q_i''(\mathbf a_{t-1}^{(0)})
\right)
\frac{\partial T_i}{\partial z_{t,j}}
\frac{\partial T_i}{\partial z_{t,k}}
\Bigg]
=
0.
\label{eq:proof-mixed-derivative}
\end{align}

For each $j\neq k$, define
\begin{equation}
\boldsymbol v_{jk}
=
\left(
\frac{\partial^2T_1}{\partial z_{t,j}\partial z_{t,k}},
\ldots,
\frac{\partial^2T_d}{\partial z_{t,j}\partial z_{t,k}},
\frac{\partial T_1}{\partial z_{t,j}}
\frac{\partial T_1}{\partial z_{t,k}},
\ldots,
\frac{\partial T_d}{\partial z_{t,j}}
\frac{\partial T_d}{\partial z_{t,k}}
\right)^\top.
\end{equation}
Evaluating Eq.~\eqref{eq:proof-mixed-derivative} for
$\ell=1,\ldots,2d$ yields
\begin{equation}
\boldsymbol L^\top\boldsymbol v_{jk}=0.
\end{equation}
Since $\boldsymbol L$ is invertible,
$\boldsymbol v_{jk}=0$. In particular,
\begin{equation}
\frac{\partial T_i}{\partial z_{t,j}}
\frac{\partial T_i}{\partial z_{t,k}}
=
0,
\qquad
\forall i,\quad j\neq k.
\label{eq:proof-no-mixing}
\end{equation}
Thus, for any fixed $i$, two distinct partial derivatives of $T_i$
cannot be nonzero simultaneously. Hence, each row of
$J_{\boldsymbol T}(\mathbf z_t)$ contains at most one nonzero entry.

\medskip
\noindent\textbf{Step III: Component-wise identifiability.}
By Step~I, $\boldsymbol r$ is locally invertible, so on every neighborhood
where $\boldsymbol T=\boldsymbol r^{-1}$ is defined, $J_{\boldsymbol T}(\mathbf z_t)$ is nonsingular. Step~II shows that each row of $J_{\boldsymbol T}(\mathbf z_t)$ contains at most one nonzero entry. Nonsingularity rules out zero rows, so each row contains exactly one nonzero entry. Since the matrix is square and
nonsingular, each column also contains exactly one nonzero entry. Thus, $J_{\boldsymbol T}(\mathbf z_t)$ is a generalized permutation matrix. Its inverse
$ J_{\boldsymbol r}(\mathbf s_t)= J_{\boldsymbol T}(\mathbf z_t)^{-1}$
is also a generalized permutation matrix. Hence, locally, each component
$r_i$ depends on exactly one latent coordinate.

Since $J_{\boldsymbol r}$ is continuous and nonsingular, its generalized-permutation pattern is locally constant. On a connected latent domain, this pattern is constant throughout the domain. Hence, there exists
a fixed permutation $\pi$ such that, for all $i=1,\ldots,d$, 
\begin{equation}
\frac{\partial r_i}{\partial s_j}=0
\quad\text{for }j\neq\pi(i),
\qquad
\frac{\partial r_i}{\partial s_{\pi(i)}}\neq0.
\label{eq:proof-fixed-permutation}
\end{equation}

On a connected product domain, Eq.~\eqref{eq:proof-fixed-permutation} implies that each $r_i$ depends only on $s_{\pi(i)}$. Thus, there exist one-dimensional functions $\rho_1,\ldots,\rho_d$ such that
\begin{equation}
z_{t,i}
=
\rho_i(s_{t,\pi(i)}),
\qquad i=1,\ldots,d.
\end{equation}
Furthermore, $\rho_i'(s_{t,\pi(i)})= \frac{\partial r_i}{\partial s_{\pi(i)}}\neq0$. Since each one-dimensional latent domain is connected and $\rho_i'$ is continuous and nowhere zero, $\rho_i$ is strictly monotone and invertible. Consequently,
\begin{equation}
\mathbf z_t
=
\left(
\rho_1(s_{t,\pi(1)}),
\ldots,
\rho_d(s_{t,\pi(d)})
\right),
\end{equation}
which establishes component-wise identifiability up to permutation and
component-wise invertible transformations.
\end{proof}

\newpage
\section{Negative-Sample Contrastive Approximation to Entropy}
\label{app:contrastive-entropy}

We provide a theoretical interpretation of the Negative-Sample contrastive term used in Eq.~\eqref{eq:gaussian-contrastive-loss}. The argument follows the kernel-density interpretation of contrastive uniformity in
\citet{wang2020understanding}, while adapting it to our action-conditioned transition likelihood.

Consider an isotropic Gaussian transition model
\begin{equation}
p_{\boldsymbol\phi}(\mathbf z_t\mid \mathbf z_{t-1},\mathbf a_{t-1} )
=
\mathcal N
\left(
\mathbf z_t;
\boldsymbol\mu_{\boldsymbol\phi}(\mathbf z_{t-1},\mathbf a_{t-1}),
\sigma^2\mathbf I
\right).
\end{equation}
Its negative log-likelihood is
\begin{equation}
\ell_{\boldsymbol\phi}(\mathbf z\mid \mathbf z_{t-1},\mathbf a_{t-1})
=
\frac{1}{2\sigma^2}
\left\|
\mathbf z-\boldsymbol\mu_{\boldsymbol\phi}(\mathbf z_{t-1},\mathbf a_{t-1})
\right\|_2^2
+
\frac d2\log(2\pi\sigma^2).
\end{equation}

\begin{proposition}[Entropy interpretation of the contrastive term]
\label{prop:contrastive-entropy}
Let
$\mathbf z^{(1)},\ldots,\mathbf z^{(K)}$
be i.i.d. samples from the marginal representation distribution
$p_{\mathbf z}$. Suppose that the transition prediction is locally accurate, such that
$\boldsymbol\mu_{\boldsymbol\phi}(\mathbf z_{t-1},\mathbf a_{t-1})=\mathbf z_t$.
Define
\begin{equation}
\mathcal U_K
=
\log
\left[
\frac1K
\sum_{k=1}^K
\exp
\left(
-\frac{
\ell_{\boldsymbol\phi}
(\mathbf z^{(k)}\mid \mathbf z_{t-1},\mathbf a_{t-1} )
}{\tau}
\right)
\right].
\end{equation}
Then, as $K\rightarrow\infty$,
\begin{equation}
\mathcal U_K
\longrightarrow
\log
\left[
(p_{\mathbf z}*\varphi_h)(\mathbf z_t)
\right]
+
C_{\sigma,\tau},
\qquad
h^2=\tau\sigma^2,
\label{eq:contrastive-kde-limit}
\end{equation}
where $\varphi_h$ is the Gaussian kernel with bandwidth $h$ and
$C_{\sigma,\tau}$ is independent of the representation distribution.
Consequently,
\begin{equation}
\mathbb E_{\mathbf z_t}[\mathcal U_K]
\longrightarrow
-H_h(\mathbf z_t)+C_{\sigma,\tau},
\end{equation}
where
\begin{equation}
H_h(\mathbf z_t)
:=
-
\mathbb E_{\mathbf z_t}
\log
\left[
(p_{\mathbf z}*\varphi_h)(\mathbf z_t)
\right]
\end{equation}
is the Gaussian-kernel resubstitution entropy.
Under the standard regularity conditions for kernel-density estimation,
as $h\rightarrow0$,
\begin{equation}
H_h(\mathbf z_t)\longrightarrow H(\mathbf z_t).
\end{equation}
Thus, up to a constant, minimizing the
negative-sample term asymptotically corresponds to maximizing the
representation entropy.
\end{proposition}

\begin{proof}
Under
$\boldsymbol\mu_{\boldsymbol\phi}(\mathbf z_{t-1},\mathbf a_{t-1} )=\mathbf z_t$,
\begin{align}
\exp
\left(
-\frac{
\ell_{\boldsymbol\phi}
(\mathbf z^{(k)}\mid\mathbf z_{t-1},\mathbf a_{t-1} )
}{\tau}
\right)
&=
C_{\sigma,\tau}'
\exp
\left(
-\frac{
\|\mathbf z^{(k)}-\mathbf z_t\|_2^2
}{
2\tau\sigma^2
}
\right)
\nonumber\\
&=
C_{\sigma,\tau}
\varphi_h(\mathbf z^{(k)}-\mathbf z_t),
\qquad
h^2=\tau\sigma^2,
\end{align}
where $C_{\sigma,\tau}$ and $C_{\sigma,\tau}'$ do not depend on
$\mathbf z_t$.

Therefore,
\begin{align}
\mathcal U_K
=
\log
\left[
\frac1K
\sum_{k=1}^K
\varphi_h(\mathbf z^{(k)}-\mathbf z_t)
\right]
+
C_{\sigma,\tau}.
\end{align}
The quantity inside the logarithm is precisely the Gaussian kernel-density
estimator of $p_{\mathbf z}$ evaluated at $\mathbf z_t$.
By the law of large numbers,
\begin{equation}
\frac1K
\sum_{k=1}^K
\varphi_h(\mathbf z^{(k)}-\mathbf z_t)
\longrightarrow
\int
\varphi_h(\mathbf z-\mathbf z_t)
p_{\mathbf z}(\mathbf z)
\,d\mathbf z
=
(p_{\mathbf z}*\varphi_h)(\mathbf z_t).
\end{equation}
This establishes Eq.~\eqref{eq:contrastive-kde-limit}.
Taking expectation over $\mathbf z_t$ gives
\[
\mathbb E[\mathcal U_K]
\longrightarrow
- H_h(\mathbf z_t)+C_{\sigma,\tau}.
\]
Finally, consistency of Gaussian kernel-density estimation gives
$H_h(\mathbf z_t)\rightarrow H(\mathbf z_t)$ as the bandwidth vanishes
under the usual regularity conditions.
\end{proof}

\newpage

\section{Implementation of A-JEPA}
\label{sec:graph-implementation}

A key practical question in A-JEPA is how to learn the temporal DAG structure in the representation space. We first introduce a predefined causal order, which allows us to incorporate DAG structure through a fully connected candidate graph consistent with this order. Ideally, with sufficient data,
a well-specified model, and global optimization, the learned transition mechanisms recover the true conditional dependencies, making non-parent candidate inputs functionally redundant. These ideal conditions are difficult to satisfy in practice, however, and redundant edges may not be clearly
suppressed under finite-sample optimization. We therefore also introduce a sparse implementation with learnable edge gates and an $\ell_1$ regularizer for explicit parent selection.
\paragraph{Predefined causal order in representation space.}
The component-wise identifiability result determines the latent causal variables only up to permutation. Consequently, the indices of the learned representation coordinates carry no predefined semantic meaning. For a latent transition graph that is acyclic over the latent components, there exists a topological ordering under which every non-self causal edge points from a lower-order component to a higher-order component. Since permutations of the learned coordinates represent equivalent identifiable solutions, we fix the coordinate order $z_1\prec\cdots\prec z_d$ as a canonical representative, and denote the ordered coordinates by $z_1,\ldots,z_d$. This choice does not assume the ground-truth semantic labels of the latent variables; it only fixes one representative within the permutation-equivalent latent space, following~\citep{JMLR:v26:22-0136}.

For example, consider three latent variables with non-self causal relations $s_3\rightarrow s_2$, $s_3\rightarrow s_1$, and $s_2\rightarrow s_1$. One valid topological ordering is $(s_3,s_2,s_1)$. Suppose the learned representation instead recovers $(s_1,s_3,s_2)$ up to component-wise invertible transformations. Since the identifiable solution is invariant to permutation, we may simply reorder the learned coordinates as $(s_3,s_2,s_1)$. Under this canonical ordering, all non-self causal edges point from lower-order to higher-order coordinates. Thus, imposing such an ordering fixes only the otherwise
arbitrary permutation of the learned representation, rather than introducing additional semantic information about the latent variables.

\paragraph{Implementation of A-JEPA.}
Based on this predefined order, we construct a fully connected candidate DAG in the representation space. For each child component $z_{t,i}$, all lower-order components $z_{t-1,j}$ with $j<i$ are treated as candidate non-self temporal parents, together with the self-transition
$z_{t-1,i}\rightarrow z_{t,i}$. Accordingly, the candidate parent set is
\begin{equation}
\mathrm{pa}_{i}
=
\{1,\ldots,i-1\}\cup\{i\}.
\label{eq:full-parent-set}
\end{equation}
Because all non-self edges follow the predefined order, the resulting graph is
acyclic by construction.

Using this full candidate parent set, the Gaussian transition model in Eq.~\eqref{eq:learned-gaussian-transition} becomes
\begin{equation}
p_{\boldsymbol{\phi}}
(\mathbf z_t\mid\mathbf z_{t-1},\mathbf a_{t-1})
=
\prod_{i=1}^{d}
\mathcal N\left(
z_{t,i};
\mu_{\boldsymbol{\phi},i}
(\mathbf z_{t-1, j \le i},
\mathbf a_{t-1}),
\sigma_{\boldsymbol{\phi},i}^2
(\mathbf a_{t-1})
\right).
\label{eq:full-dag-gaussian-transition}
\end{equation}

\begin{lstlisting}[
caption={A-JEPA with a fully connected candidate DAG.},
label={alg:ajepa-full-dag},
basicstyle=\ttfamily\small,
frame=single
]
for each minibatch (x_{t-1}, a_{t-1}, x_t):
    z_{t-1} = h(x_{t-1})
    z_t     = h(x_t)
    for i = 1, ..., d:
        # candidate parents follow the predefined order
        parent_i = z_{t-1, 1:i}
        mu_i    = mean_phi_i(parent_i, a_{t-1})
        sigma_i = std_phi_i(a_{t-1})
    # positive transition likelihood
    nll_pos = GaussianNLL(z_t,{mu_i, sigma_i}_{i=1}^d)
    # sample alternative future representations
    {z_t^(k)}_{k=1}^K = SampleNegatives(z_t)
    for k = 1, ..., K:
        nll_neg[k] = GaussianNLL(z_t^(k),{mu_i, sigma_i}_{i=1}^d)
    # A-JEPA objective
    L = alpha*nll_pos + (1-alpha)*log((1/K)*sum_k exp(-nll_neg[k]/tau))
    update h, phi by minimizing L
\end{lstlisting}

\paragraph{Implementation of A-JEPA with Sparse Parent Selection.}
In practice, finite samples and imperfect optimization make it difficult to
determine whether a fitted transition mechanism is functionally independent of
a redundant candidate parent. We therefore make the parent selection explicit
by introducing a learnable scalar gate $g_{ji}$ for each candidate non-self
edge $z_{t-1,j}\rightarrow z_{t,i}$ with $j<i$. The self-transition
$z_{t-1,i}\rightarrow z_{t,i}$ is always retained.

For the $i$-th transition mechanism, we define the gated parent input as
\begin{equation}
\widetilde{\mathbf z}_{t-1,1:i}^{(i)}
=
\left(
g_{1i}z_{t-1,1},
\ldots,
g_{(i-1)i}z_{t-1,i-1},
z_{t-1,i}
\right).
\label{eq:gated-parent-input}
\end{equation}
The Gaussian transition model is then implemented as
\begin{equation}
p_{\boldsymbol{\phi}}
(\mathbf z_t\mid\mathbf z_{t-1},\mathbf a_{t-1})
=
\prod_{i=1}^{d}
\mathcal N\left(
z_{t,i};
\mu_{\boldsymbol{\phi},i}
(\widetilde{\mathbf z}_{t-1,1:i}^{(i)},\mathbf a_{t-1}),
\sigma_{\boldsymbol{\phi},i}^2
(\mathbf a_{t-1})
\right).
\label{eq:sparse-gated-gaussian-transition}
\end{equation}

To encourage sparse parent selection, we augment the practical A-JEPA
objective in Eq.~\eqref{eq:gaussian-contrastive-loss} with an $\ell_1$
penalty on the non-self gates:
\begin{equation}
\mathcal L_{\mathrm{sparse}}
=
\mathcal L
+
\lambda_{\mathrm{sp}}
\sum_{i=1}^{d}
\sum_{j<i}
|g_{ji}|,
\label{eq:sparse-ajepa-objective}
\end{equation}
where $\lambda_{\mathrm{sp}}>0$ controls the sparsity strength. After
optimization, the estimated non-self parent set of $z_{t,i}$ is obtained by
thresholding the learned gates:
\begin{equation}
\widehat{\mathrm{pa}}_i^z
=
\{j<i:\ |g_{ji}|>\delta\}
\cup\{i\},
\label{eq:estimated-parent-set}
\end{equation}
where $\delta$ is a fixed threshold.

\begin{lstlisting}[
caption={A-JEPA with sparse parent selection.},
label={alg:ajepa-sparse-dag},
basicstyle=\ttfamily\small,
frame=single
]
for each minibatch (x_{t-1}, a_{t-1}, x_t):
    z_{t-1} <- h(x_{t-1})
    z_t     <- h(x_t)
    for i = 1, ..., d:
        P_i <- (g_{1i} z_{t-1,1},...,g_{(i-1)i} z_{t-1,i-1},z_{t-1,i})
        mu_i    <- mean_phi_i(P_i, a_{t-1})
        sigma_i <- std_phi_i(a_{t-1})
    L_pos <- GaussianNLL(z_t | z_{t-1}, a_{t-1})
    {z_t^(k)}_{k=1}^K <- sample alternative futures
    for k = 1, ..., K:
        L_neg^(k) <- GaussianNLL(z_t^(k) | z_{t-1},a_{t-1})
    L <- alpha*L_pos + (1-alpha) * log[(1/K) * sum_k exp(-L_neg^(k) / tau)]
         + lambda_sp * sum_{i} sum_{j<i} |g_{ji}|
    update h, phi, and {g_{ji}} by minimizing L
for each i:
    pa_hat_i <- {j < i : |g_{ji}| > delta} union {i}
\end{lstlisting}

\section{Simulation Details}
\label{app:simulation-details}

This section provides additional details of the synthetic experiments. The data are generated according to the action-modulated Gaussian ANM in
Sec.~\ref{sec:case-study}. The simulation is designed to control the main factors appearing in our theory, including the observation mapping and
action-induced variation in the latent transition mechanisms.

\paragraph{Latent transition model.}
We generate a latent causal state $\mathbf s_t\in\mathbb R^d$ following a temporal DAG. For each random seed, the DAG is sampled once and kept fixed across all episodes. Each latent variable has a self-transition and may have additional parents from the previous state, following the known causal order. Each variable has at most three temporal parents including itself. The transition follows
\begin{equation}
s_{t,i} = f_i(\mathbf s_{t-1,\mathrm{pa}_i},\mathbf a_{t-1}^{\mu}) + \sigma_i(\mathbf a_{t-1}^{\sigma})\epsilon_{t,i},
\qquad \epsilon_{t,i}\sim\mathcal N(0,1),
\end{equation}
where
\begin{equation}
\mathbf a_{t-1} = [\mathbf a_{t-1}^{\mu},\mathbf a_{t-1}^{\sigma}], \qquad \mathbf a_{t-1}^{\mu},\mathbf a_{t-1}^{\sigma}\in\mathbb R^d.
\end{equation}
Thus, the default action dimension is $2d$. The first action block modulates the conditional mean, while the second modulates the conditional variance.

For a temporal causal edge $j\rightarrow i$, its action-dependent coefficient is
\begin{equation}
A_{ji}(\mathbf a_{t-1}^{\mu})= \kappa G_{ji} \tanh\!\left( \mathbf w_{ji}^{\top}\mathbf a_{t-1}^{\mu}
\right), \qquad \kappa=0.5.
\end{equation}
Self-transitions additionally have a fixed coefficient $\rho=0.6$. The resulting mean is passed through a LeakyReLU nonlinearity. Action coordinates not associated with an active parent mechanism are masked
out, so that the conditional mean is modulated only through the sampled
temporal causal edges. The conditional standard deviation is
\begin{equation}
\sigma_i(\mathbf a_{t-1}^{\sigma}) = \sigma_{\min} + s_{\sigma}\, \mathrm{softplus}(a_{t-1,i}^{\sigma}),
\end{equation}
where $\sigma_{\min}=0.08$ and $s_{\sigma}=1.0$.

\paragraph{Action settings.}
The action is fixed within each episode and varied across episodes. Therefore, different episodes correspond to different action-conditioned transition mechanisms. Unless otherwise specified, each episode contains $2{,}000$ samples and we generate $10d$ episodes. To study the effect of action variability, we vary the number of action settings. For $d=5$, we use $5$, $10$, $20$, $30$, $40$, and $50$ different action settings.

\paragraph{Observation model.}
The latent state is mapped to the observation as
\begin{equation}
\mathbf x_t=\boldsymbol{\mathrm g}(\mathbf s_t),
\end{equation}
where $\boldsymbol{\mathrm g}$ is constructed using random orthogonal transformations and LeakyReLU nonlinearities. The default mapping is
invertible. To examine the role of this assumption, we progressively remove observed coordinates while keeping the latent dimension fixed, making the observation mapping increasingly non-invertible.

\paragraph{Network architecture.}
The used encoder is a three-layer MLP with hidden width $256$, where each hidden block consists of Linear-BatchNorm-SiLU layers, followed by a linear projection to the $d$-dimensional representation space. The encoder output is mapped coordinate-wise to $(\delta,1-\delta)$ using a sigmoid transform with $\delta=10^{-4}$. The used transition model predicts a factorized Gaussian distribution over the learned coordinates. For each coordinate $i$, the conditional mean is produced by a separate two-layer MLP with hidden width $256$, taking as input the previous representations of its candidate temporal parents together with the mean-action variables. The conditional variance is predicted by a separate coordinate-wise MLP from the action variables, with log-variance clamped to $[-1,1]$.

\paragraph{Model and training.}
The encoder maps $\mathbf x_t$ to a learned representation $\mathbf z_t$. The learned transition model predicts a diagonal Gaussian distribution
\begin{equation}
p_{\boldsymbol\phi} (\mathbf z_t\mid\mathbf z_{t-1},\mathbf a_{t-1}) = \prod_{i=1}^{d} \mathcal N\left( z_{t,i}; \mu_{\boldsymbol\phi,i}, \sigma_{\boldsymbol\phi,i}^{2}\right),
\end{equation}
where the conditional mean and variance depend on the previous representation and action. In the main representation-recovery experiments, A-JEPA uses the
maximally connected candidate DAG consistent with the predefined causal
order, as described in Appendix~\ref{sec:graph-implementation}. For graph
recovery, we use its sparse-gated variant and learn the active non-self
parents from data.

We optimize the Gaussian contrastive objective in Eq.~\eqref{eq:gaussian-contrastive-loss}. Unless otherwise specified, we use $\alpha=0.51$, $\tau=1.0$, Adam with learning rate $10^{-4}$, weight decay $10^{-5}$, batch size $512$, and $200$ training epochs. We use $80\%$ of the
generated transitions for training and $20\%$ for testing.

\paragraph{Evaluation.}
Latent state recovery is measured using component-wise Spearman MCC. We compute the absolute Spearman correlations between all learned and ground-truth latent coordinates, align them using the Hungarian algorithm, and average the matched correlations. This metric is consistent with the component-wise equivalence class in Theorem~\ref{thm:componentwise-identifiability}.

For graph recovery, we first align the learned coordinates using the same matching. We then threshold the learned parent gates and report precision, recall, F1, and structural Hamming distance (SHD) against the ground-truth temporal DAG. Self-transitions are excluded from graph evaluation because they are always included in the simulation.

\paragraph{Assumption and model-specification analysis.}
We consider three controlled changes to the default setting. First, we reduce action-induced variation by allowing actions to modulate only the conditional mean or only the conditional variance. Second, we replace Gaussian noise with Laplace, Student-$t$, or Uniform noise while keeping the same action-dependent variance. Third, we progressively remove observation coordinates to weaken the invertibility of the observation mapping. These experiments examine how latent-state recovery changes when the main assumptions or modeling choices are relaxed.

\paragraph{Computational resources.}
All experiments were conducted on NVIDIA A100 GPUs with 40\,GB of memory, with different random seeds and experimental settings executed independently.

\section{Visual Benchmark Details}
\label{app:visual-benchmark-details}

This section provides additional details of the four visual control benchmarks
used in Sec.~\ref{sec:exp}. These environments provide image observations and
ground-truth physical state variables.
We therefore use them to evaluate whether the learned representations align
with interpretable state factors, preserve information about the underlying
state, and support action-conditioned latent prediction.

\paragraph{Encoder and transition architecture.}
For all visual benchmarks, A-JEPA and the baselines use the same compact CNN encoder. Given an RGB observation resized to $224\times224$, the encoder applies four convolutional blocks with channel widths $32$, $64$, $128$, and $256$. Each block consists of a $5\times5$ convolution with stride $2$ and padding $2$, followed by BatchNorm and LeakyReLU. The resulting feature map is globally average pooled and projected to the $d_z$-dimensional representation space through a two-layer MLP with hidden width $512$. Following the data convention used by the predictive baselines, we use a frame skip of $5$. The five primitive actions between two encoded frames are concatenated and mapped to a $d_z$-dimensional action representation through a small action encoder.

A-JEPA uses a coordinate-wise transition model following a predefined latent order, which is without loss of generality due to the permutation indeterminacy in latent space, see the discussion in Sec.~\ref{sec:graph-implementation}. For each coordinate $i$, the conditional mean is predicted from the previous latent coordinates up to $i$ together with the encoded action,
\begin{equation}
\mu_{t+1,i}
=
f_i\!\left(
\mathbf z_{t,\le i},
\mathbf a_t^{z}
\right),
\end{equation}
where each $f_i$ is a separate two-layer MLP with hidden width $256$. The transition variance is predicted by a separate action-conditioned head, with one log-variance per latent coordinate and values clamped to $[-1,1]$. Together, these heads define the factorized Gaussian transition model used by the A-JEPA objective.

\paragraph{Baselines and training.}
All methods take image observations as input and use the same CNN encoder and
comparable training setup. A-JEPA is trained with the objective in Eq.~\eqref{eq:gaussian-contrastive-loss} and the action-conditioned transition model described above. LeWM and PLDM use the same image encoder, action representation, latent dimension, frame skip, optimizer settings, batch size, and number of training epochs, while retaining their original predictive
objectives and non-DAG transition models.

Unless otherwise specified, models are trained from scratch for $10$ epochs using Adam with learning rate $10^{-4}$, weight decay $10^{-5}$, batch size $128$, and history size $1$. We reserve $10\%$ of the official training split for validation and evaluate on the held-out test split. For A-JEPA, we use $\alpha=0.51$ and $\tau=1.0$. All reported results are averaged over three random seeds.

\paragraph{Evaluation metrics.}
We evaluate the learned representations from three complementary perspectives.
Component-wise MCC measures alignment between individual learned coordinates
and ground-truth physical state factors. We compute absolute Spearman
correlations, perform Hungarian matching, and average the matched scores. This
metric is consistent with the component-wise equivalence class in
Theorem~\ref{thm:componentwise-identifiability}.

Nonlinear state $R^2$ measures how much information about the underlying state
is retained in the full representation, regardless of coordinate-wise
alignment. Latent rollout $R^2$ evaluates action-conditioned prediction by
recursively applying the learned transition model and comparing the predicted
future representation with the encoder representation of the corresponding
future observation.

\paragraph{TwoRoom.}
\textsc{TwoRoom} is a two-dimensional navigation environment where the action-driven state is the agent position,
\begin{equation}
\mathbf s_t=[x_t,y_t].
\end{equation}
The dataset observation also contains target coordinates, but we exclude them from MCC because they specify the episode goal/context rather than the evolving state. Concretely, we evaluate component recovery only on the two dynamic
coordinates $(x_t,y_t)$.
\begin{figure}[h]
    \centering
    \includegraphics[width=0.8\linewidth]{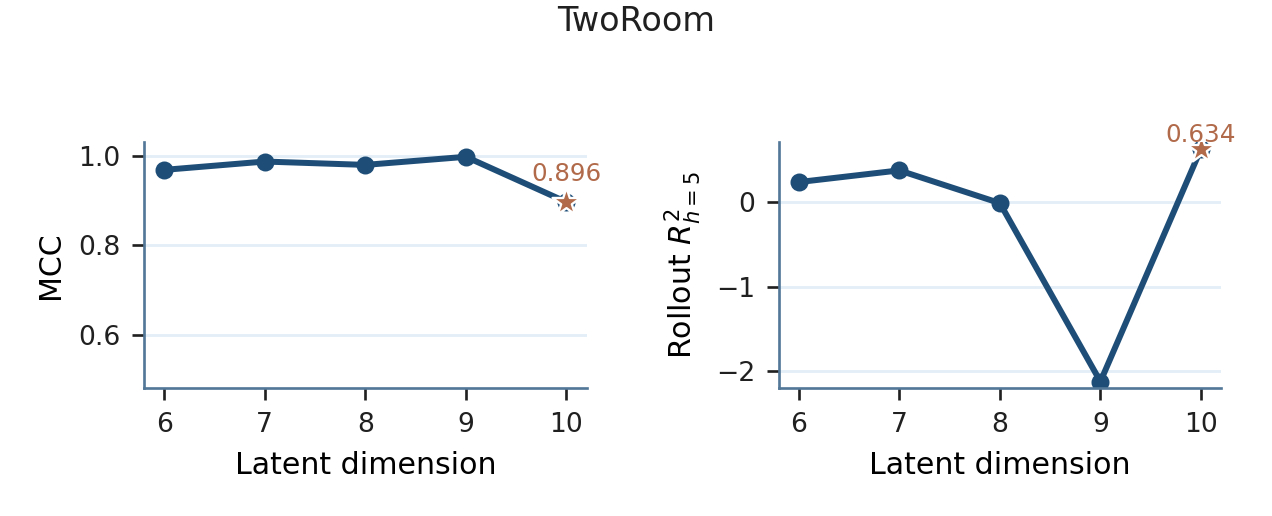}
    \caption{
    Latent-dimension selection on \textsc{TwoRoom}. We sweep
    $d_z\in\{6,7,8,9,10\}$ and evaluate both component-wise recovery
    (MCC) and rollout $R^2$ at horizon $5$. We use $d_z=10$ for the
    main experiments because it provides a strong balance between
    component recovery and multi-step transition prediction.
    }
    \label{fig:tworoom-zdim}
\end{figure}
For evaluation, we use $5{,}000$ held-out test samples for each seed and compute the absolute Spearman correlation between every learned latent coordinate and each ground-truth state factor. We then apply Hungarian matching before averaging the matched scores. This accounts for both the permutation ambiguity and the component-wise monotone transformations allowed by our identifiability result. We sweep the latent dimension over $d_z\in\{6,7,8,9,10\}$. As shown in Figure~\ref{fig:tworoom-zdim}, although smaller dimensions achieve high MCC, the $d_z=10$ model gives the best rollout behavior among the sweep settings, with rollout $R^2_{h=5}=0.6343$ in the dimension-selection run. We therefore use $d_z=10$ for the main three-seed comparison in Figure~\ref{fig:visual-control-main}, selecting the dimension by the joint criterion of component recovery and action conditioned dynamics quality.

\paragraph{Reacher.}
\textsc{Reacher} is a two-link articulated arm with two rotational state variables. Because joint angles are periodic, direct correlation with the raw angles can be misleading near the $-\pi/\pi$ boundary. Thus, we represent each angle $q_j$ by its circular coordinates $(\sin q_j,\cos q_j)$. We measure component recovery over the two physical angular factors, namely the proximal joint angle $q_0$ and the distal joint angle $q_1$, treating each $(\sin q_j,\cos q_j)$ pair as one circular state factor rather than two independent scalar factors. For each learned coordinate $z_i$, we compute its best one-dimensional projection onto $(\sin q_j,\cos q_j)$ and use the resulting correlation for Hungarian matching across the two joint-angle factors. Evaluation is performed on $5{,}000$ held-out test samples.
\begin{figure}[h]
    \centering
    \includegraphics[width=0.8\linewidth]{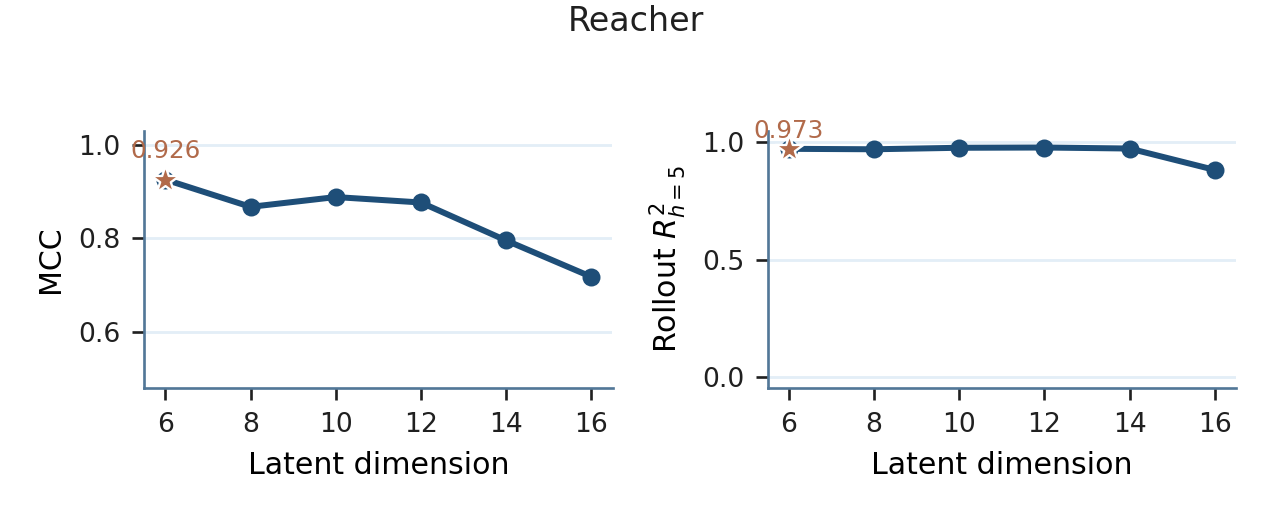}
    \caption{
    Latent-dimension selection on \textsc{Reacher}. We sweep
    $d_z\in\{6,8,10,12,14,16\}$ and evaluate grouped circular MCC and
    rollout $R^2$ at horizon $5$. We use $d_z=6$ for the main experiments
    because it achieves strong component-wise recovery and multi-step
    prediction with the most compact representation.
    }
    \label{fig:reacher-zdim}
\end{figure}
We sweep the latent dimension over $d_z\in\{6,8,10,12,14,16\}$ and select the representation size by jointly considering component-wise recovery and multi-step prediction. As shown in Fig.~\ref{fig:reacher-zdim}, $d_z=6$ achieves strong performance on both criteria, with MCC $0.926$ and rollout $R^2_{h=5}=0.973$, while increasing
the latent dimension does not improve component recovery. We therefore use
$d_z=6$ for the main three-seed comparison.

\paragraph{PushT.}
\textsc{PushT} is a planar manipulation environment in which a controlled pusher moves a T-shaped block through contact. We evaluate component recovery using five physical state factors,
\begin{equation}
\mathbf s_t=
[x_t^p,y_t^p,x_t^b,y_t^b,\theta_t^b],
\end{equation}
where $p$ and $b$ denote the pusher and block, respectively. Since pusher motion is directly driven by the action, we additionally include two finite-difference motion factors,
\begin{equation}
\Delta x_t^p=x_{t+1}^p-x_t^p,
\qquad
\Delta y_t^p=y_{t+1}^p-y_t^p,
\end{equation}
computed from adjacent pusher positions. The main MCC is evaluated over seven factors:
\begin{equation}
[x_t^p,y_t^p,x_t^b,y_t^b,\theta_t^b,
\Delta x_t^p,\Delta y_t^p].
\end{equation}
\begin{figure}[h]
    \centering
    \includegraphics[width=0.8\linewidth]{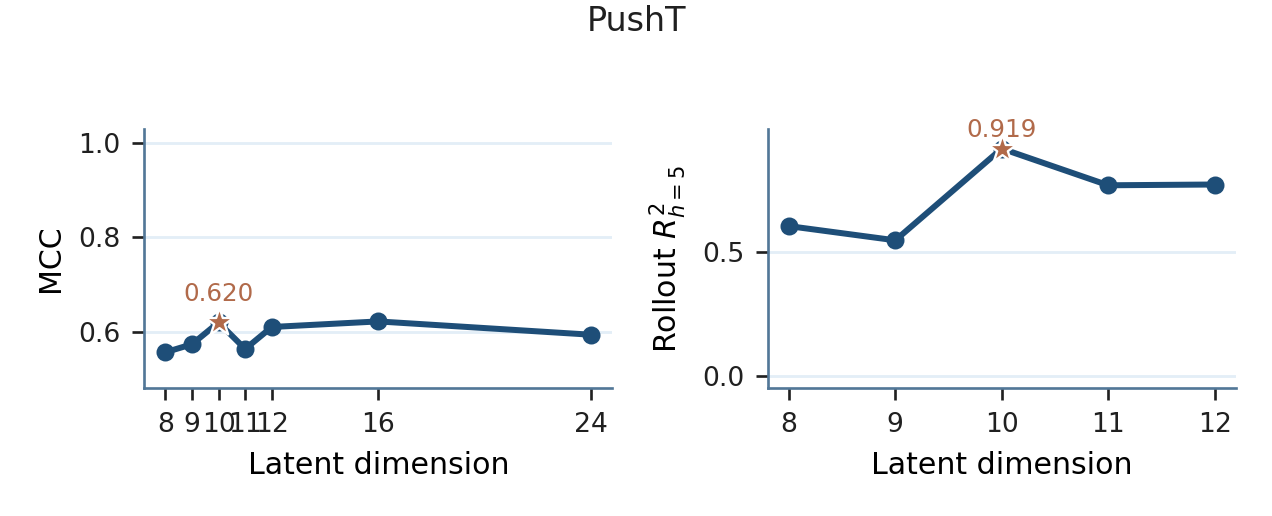}
    \caption{
    Latent-dimension selection on \textsc{PushT}. We evaluate both
    component-wise recovery and rollout $R^2$ at horizon $5$.
    We use $d_z=10$ for the main experiments because it provides the
    strongest joint component-recovery and multi-step prediction performance
    in the dimension-selection run.
    }
    \label{fig:pusht-zdim}
\end{figure}
Again, we select the representation dimension by jointly considering component-wise recovery and multi-step prediction. As shown in Fig.~\ref{fig:pusht-zdim}, the dimension sweep includes $d_z\in\{8,9,10,11,12,16,24\}$ for MCC, with rollout evaluation concentrated around the smaller dimensions. The $d_z=10$ setting provides the strongest joint performance in the selection run, achieving MCC $0.620$ and rollout $R^2_{h=5}=0.919$. We use $d_z=10$ for the main three-seed comparison reported in Figure~\ref{fig:visual-control-main}.

\paragraph{OGBench Cube.}
\textsc{OGBench Cube} is a robot manipulation environment involving end-effector motion, gripper actuation, and object manipulation. We evaluate
component recovery on seven task-relevant physical factors,
\begin{equation}
\mathbf s_t=
[
x_t^{ee},y_t^{ee},z_t^{ee},
x_t^{obj},y_t^{obj},z_t^{obj},
g_t
],
\end{equation}
where $ee$ denotes the end-effector, $obj$ denotes the manipulated object, and $g_t$ denotes the gripper opening. We use these factors rather than the full raw state vector because they directly describe the visible and action-relevant state of the manipulation task. For each seed, evaluation is performed on $5{,}000$ held-out test samples.
\begin{figure}[h]
    \centering
    \includegraphics[width=0.8\linewidth]{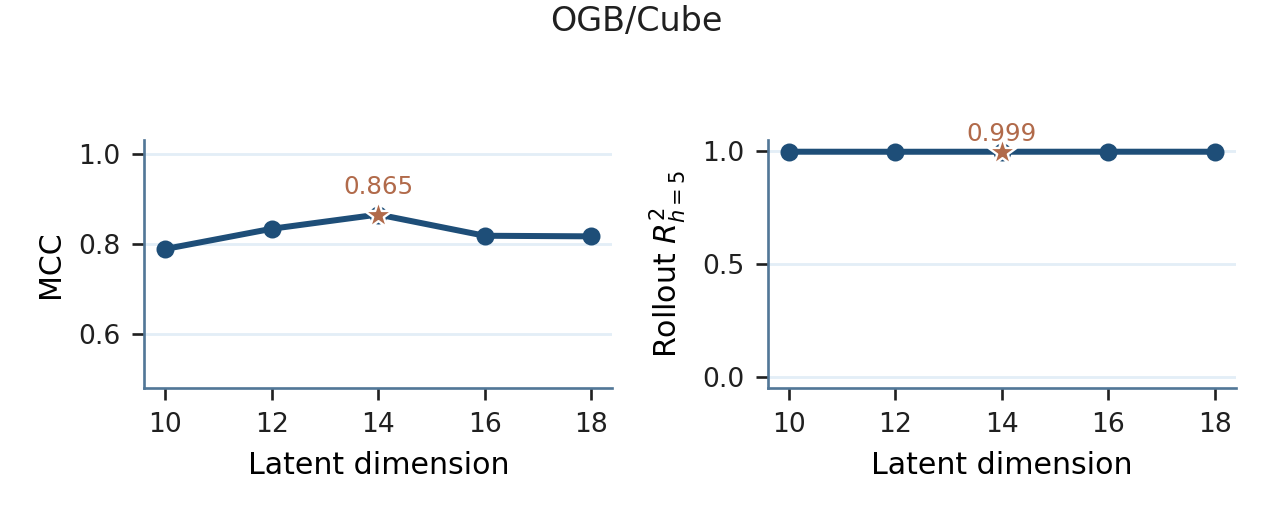}
    \caption{
    Latent-dimension selection on \textsc{OGBench Cube}. We sweep
    $d_z\in\{10,12,14,16,18\}$ and evaluate component-wise recovery and
    rollout $R^2$ at horizon $5$. We use $d_z=14$ for the main experiments
    because it achieves the strongest joint recovery and multi-step prediction
    performance.
    }
    \label{fig:ogb-zdim}
\end{figure}
We sweep the latent dimension over $d_z\in{10,12,14,16,18}$ and select the representation size by jointly considering component-wise recovery and multi-step prediction. As shown in Fig.~\ref{fig:ogb-zdim}, $d_z=14$ achieves the strongest overall performance, with MCC $0.865$ and rollout $R^2_{h=5}=0.999$, and is therefore used for the main three-seed comparison. At this dimension, A-JEPA strongly recovers the end-effector coordinates and horizontal object position, while object height and gripper opening remain more challenging. The learned transition is also highly stable, with rollout $R^2=0.9987$ at horizon $5$ and $0.9940\pm0.0051$ at horizon $10$.

\section{Visual Intervention Experiments on \textsc{Reacher}}
\label{app:visual-intervention}
\begin{figure}[t]
    \centering
    \includegraphics[width=0.5\linewidth]
    {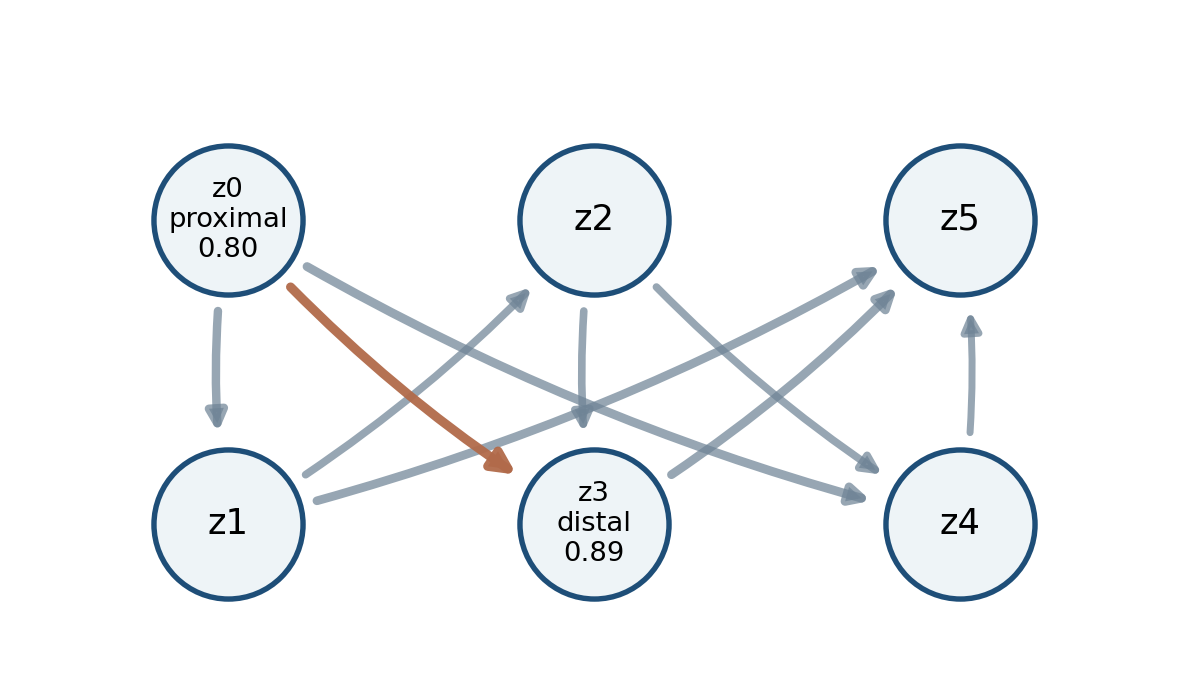}
    \caption{
    Learned latent transition graph on \textsc{Reacher}. MCC assigns $z_0$ to the proximal joint factor and $z_3$ to the distal joint
    factor. The learned graph contains the directed edge $z_0\rightarrow z_3$, consistent with the proximal-to-distal dependency of
    the two-link arm.
    }
    \label{fig:reacher-latent-graph}
\end{figure}
We further examine whether the component-wise semantics and directed latent dependencies learned by A-JEPA correspond to meaningful physical changes in the visual environment. We use \textsc{Reacher} for this analysis because its two-link articulated structure provides a simple setting in which the direction of physical dependencies can be interpreted clearly.

Our evaluation follows three steps. First, we assign physical semantics to individual learned coordinates using the grouped circular MCC described in Sec.~\ref{app:visual-benchmark-details}. Second, we inspect whether the directed edges learned between these coordinates are consistent with the physical structure of the system. Third, we intervene on the corresponding latent coordinates and propagate the intervention through the learned transition model to examine whether the resulting visual changes follow the
learned direction of dependence.
\begin{figure*}[t]
    \centering
    \includegraphics[width=\textwidth]
    {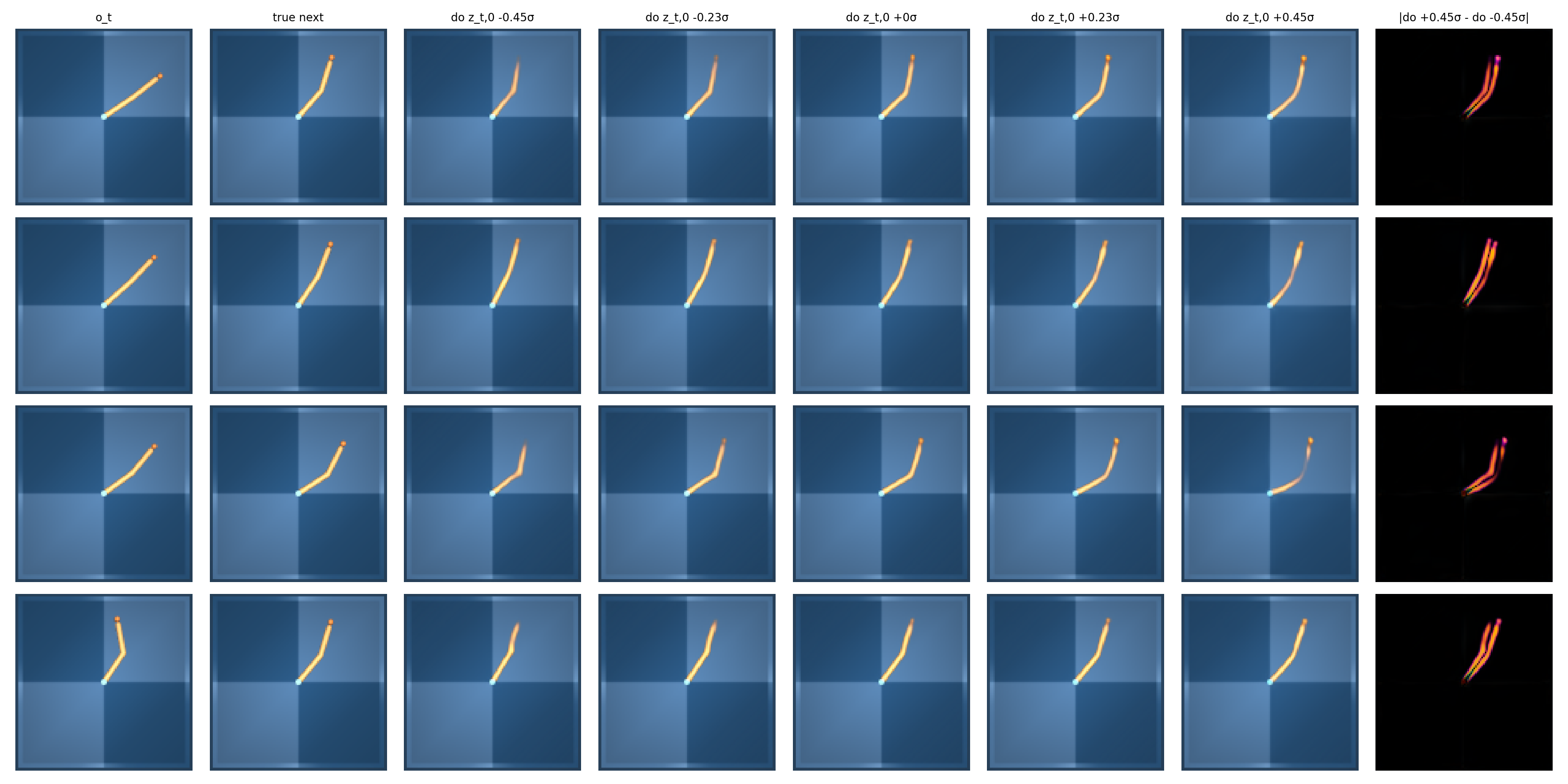}
    \vspace{1.5mm}
    \includegraphics[width=\textwidth]
    {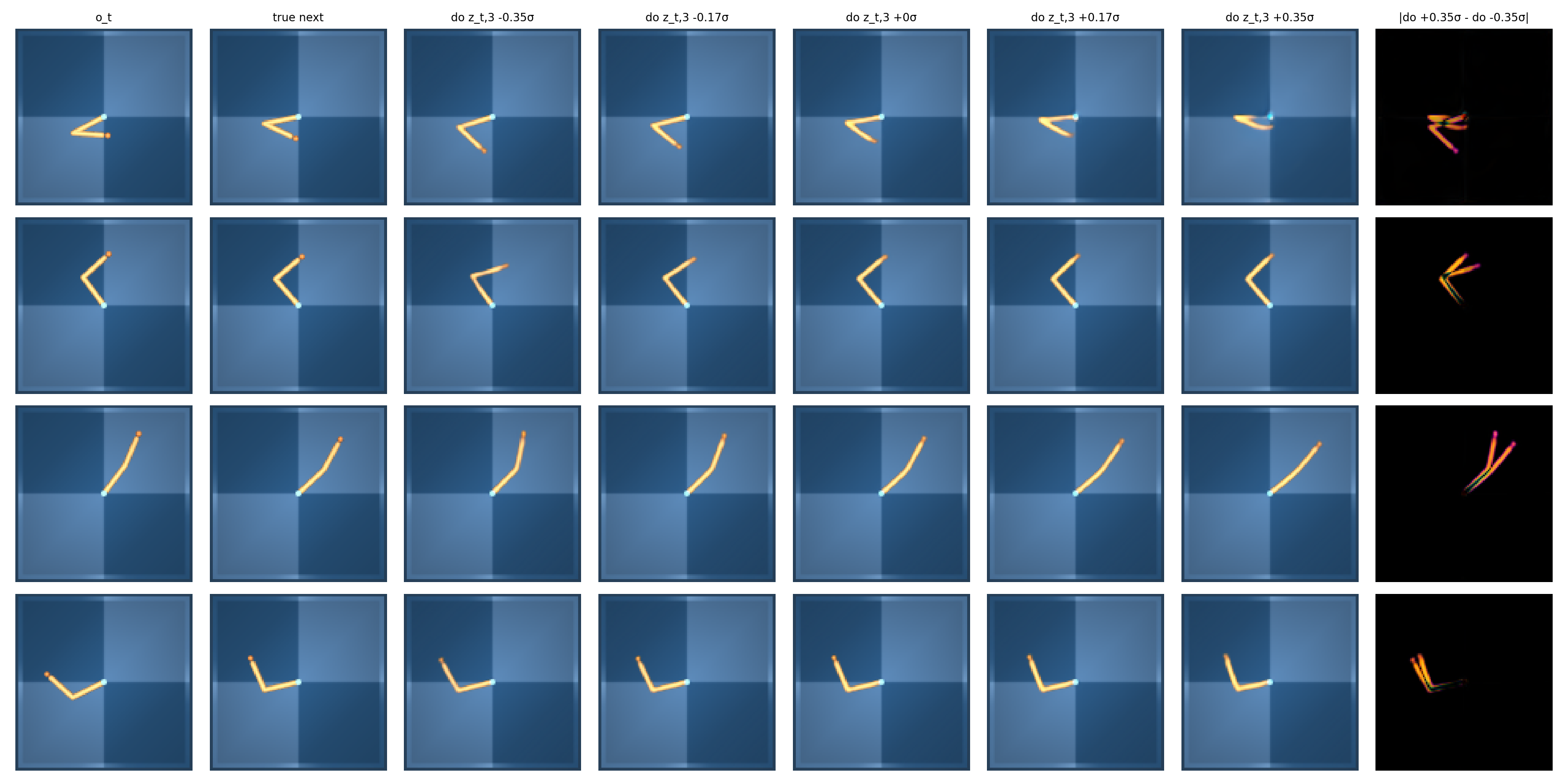}
    \caption{ Latent intervention diagnostics on \textsc{Reacher}. Top: intervention on the proximal-aligned coordinate $z_0$ changes the
    proximal joint and induces corresponding changes in the downstream distal arm segment. Bottom: intervention on the distal-aligned coordinate $z_3$ produces a more localized distal-joint change while leaving the proximal configuration relatively unchanged. This asymmetric response is consistent with the learned directed dependency $z_0\rightarrow z_3$.
    }
    \label{fig:reacher-do-interventions}
\end{figure*}
\paragraph{Semantic alignment.}
The physical state of \textsc{Reacher} contains two angular factors: the proximal joint angle and the distal joint angle. Using grouped circular MCC, the learned coordinate $z_0$ is matched to the proximal joint factor, while
$z_3$ is matched to the distal joint factor. Their corresponding matching scores are approximately $0.80$ and $0.89$ shown in Figure~\ref{fig:reacher-latent-graph}, respectively. This semantic assignment is performed before examining the learned graph or intervention results.

\paragraph{Learned directed dependency.}
After assigning semantics to the latent coordinates, we inspect the learned gated transition graph. As shown in Fig.~\ref{fig:reacher-latent-graph}, the model identifies a directed edge $z_0 \rightarrow z_3$, corresponding to a proximal-to-distal dependency. This relation is consistent with the physical structure of a two-link arm: changing the proximal joint changes the configuration of the upstream link and consequently affects the position and orientation of the downstream link, whereas the distal joint
does not analogously determine the proximal joint.

\paragraph{Latent interventions.}
To further examine the direction of this learned dependency, we intervene on one latent coordinate at time $t$ while keeping the remaining coordinates fixed. For a selected coordinate $z_i$, we consider interventions of the form $z_{t,i}^{\mathrm{do}}=z_{t,i}
+\delta\,\mathrm{Std}(z_i)$, where $\delta$ controls the intervention magnitude. The intervened latent state is then propagated through the learned action-conditioned transition model, $\hat{\mathbf z}_{t+1}^{\mathrm{do}}=T_{\boldsymbol\phi}(\mathbf z_t^{\mathrm{do}},\mathbf a_t)$, and the predicted next representation is decoded back to image space. Thus, the visualization reflects the effect of the intervention after propagation through the learned transition dynamics, rather than simply decoding a manually perturbed latent coordinate.

\paragraph{Directional intervention results.}
The intervention responses are consistent with the learned edge $z_0\rightarrow z_3$. As shown in Figure~\ref{fig:reacher-do-interventions}, when intervening on the proximal-aligned coordinate
$z_0$, changing its value produces a clear change in the proximal joint and also induces a corresponding change in the distal arm segment. In contrast, intervening on the distal-aligned coordinate $z_3$ primarily changes the distal joint while leaving the proximal configuration largely unchanged. This asymmetric response is consistent with the learned proximal-to-distal
direction: intervention on the parent affects both the parent and its downstream child, whereas intervention on the child produces a more localized change.

Taken together, these experiments provide three complementary checks on the learned latent structure: MCC assigns interpretable physical semantics to individual coordinates, the learned graph recovers a physically plausible directed dependency between them, and latent interventions produce visual changes consistent with this direction. These results provide qualitative support that the learned transition structure captures meaningful directed dependencies rather than only preserving information about the physical state.

\section{Details of Visual Transfer Experiments}
\label{app:visual-transfer-details}

This section provides additional details of the held-out robot transfer experiment on RoboNet~\citep{dasari2019robonet}. The experiment evaluates whether action-conditioned latent dynamics learned from a set of robot platforms remain predictive on robots not observed during training.

\paragraph{Dataset and split.}
RoboNet contains image-based manipulation data collected from multiple robot platforms. We train all models on Baxter, Kuka, and WidowX and evaluate them without adaptation on held-out Franka and Sawyer robots. Images are resized to
$224\times224$, and all selected subsets use a four-dimensional action input. The source training set contains $9{,}300$ transition samples. We restrict the experiment to robot subsets with compatible action interfaces so that transfer
can be evaluated without introducing additional action mappings.

\paragraph{Models and training.}
For the RoboNet transfer experiments, A-JEPA uses the same compact CNN encoder as in the visual benchmarks, with latent dimension $8$. The encoder contains four convolutional blocks with channel widths $32$, $64$, $128$, and $256$. Each block applies a $5\times5$ convolution with stride $2$ and padding $2$, followed by BatchNorm and $\mathrm{LeakyReLU}(0.1)$. The resulting feature map is globally average pooled and projected to the latent space through
\[
\mathrm{Linear}(256,512)
\rightarrow
\mathrm{LeakyReLU}(0.1)
\rightarrow
\mathrm{Linear}(512,8).
\]

The four-dimensional action input is mapped to the latent dimension using the same action encoder as in the visual benchmark experiments. A-JEPA uses a coordinate-wise transition model following the implementation in Sec.~\ref{sec:graph-implementation}. For coordinate $i$, the conditional mean is predicted from the ordered latent
coordinates up to $i$ together with the encoded action, $\mu_{t+1,i}=f_i(\mathbf z_{t,\leq i},\mathbf a_t^{z}),$ where each $f_i$ is a two-layer MLP with hidden width $256$. A separate
action-conditioned head predicts one log-variance per latent coordinate for the
factorized Gaussian transition model. All models are trained for $5$ epochs using AdamW with learning rate $10^{-4}$, weight decay $10^{-5}$, batch size $64$, and history size $1$. For A-JEPA, we use $\alpha=0.51$ and $\tau=1.0$. Results are averaged over three random seeds.

\paragraph{Baselines.}
We compare A-JEPA with matched-capacity LeWM and PLDM baselines. All methods use the same CNN encoder, latent dimension $8$, action input, history size, optimizer, batch size, and number of training epochs. LeWM and PLDM retain their respective predictive objectives and non-DAG transition predictors; they are not given the coordinate-wise transition structure or the
Gaussian-contrastive A-JEPA objective.

\paragraph{Evaluation protocol.}
At test time, all model parameters are frozen. Given the encoded state $\mathbf z_t$ of a held-out robot and its subsequent action sequence, we recursively apply the learned transition model to obtain
$\hat{\mathbf z}_{t+h}$. At each horizon $h$, we compare the predicted latent state with the encoded future observation $\mathbf z_{t+h}=\boldsymbol{\mathrm h}(\mathbf x_{t+h})$ using rollout $R^2$, where larger values indicate more accurate latent prediction. We report rollout $R^2$ over horizons $h=1,\ldots,20$ and its average across horizons.
\begin{table}[t]
\centering
\caption{Held-out robot transfer on RoboNet. Models are trained on Baxter, Kuka, and
WidowX and evaluated without adaptation on Franka and Sawyer. We report latent
rollout $R^2$ at selected horizons and the average over $h=1,\ldots,20$.
Values are mean $\pm$ standard deviation over three seeds.
}
\label{tab:robonet-transfer}
\begin{tabular}{lccc}
\toprule
Horizon & A-JEPA & LeWM & PLDM \\
\midrule
\multicolumn{4}{c}{\textbf{Held-out Franka}} \\
$h=1$ & $0.9691\pm0.0088$ & $0.6154\pm0.2418$ & $0.4486\pm0.4487$ \\
$h=2$  & $0.9605\pm0.0112$ & $0.6033\pm0.1973$ & $-0.0679\pm0.4816$ \\
$h=3$  & $0.9534\pm0.0128$ & $0.5728\pm0.1905$ & $-0.8018\pm1.1256$ \\
$h=4$  & $0.9468\pm0.0141$ & $0.5428\pm0.1928$ & $-1.7901\pm2.0465$ \\
$h=6$  & $0.9387\pm0.0158$ & $0.4878\pm0.1995$ & $-3.9602\pm3.8622$ \\
$h=7$  & $0.9353\pm0.0165$ & $0.4588\pm0.2028$ & $-4.8820\pm4.5269$ \\
$h=8$  & $0.9314\pm0.0176$ & $0.4261\pm0.2080$ & $-5.6401\pm5.0212$ \\
$h=9$  & $0.9282\pm0.0181$ & $0.3948\pm0.2155$ & $-6.2633\pm5.3929$ \\
$h=11$ & $0.9247\pm0.0192$ & $0.3278\pm0.2234$ & $-7.2379\pm5.9311$ \\
$h=12$ & $0.9227\pm0.0187$ & $0.2931\pm0.2235$ & $-7.6042\pm6.1337$ \\
$h=13$ & $0.9210\pm0.0186$ & $0.2582\pm0.2223$ & $-7.8627\pm6.2301$ \\
$h=14$ & $0.9192\pm0.0191$ & $0.2197\pm0.2193$ & $-8.0768\pm6.3313$ \\
$h=15$ & $0.9170\pm0.0191$ & $0.1802\pm0.2132$ & $-8.1900\pm6.3764$ \\
$h=16$ & $0.9146\pm0.0183$ & $0.1445\pm0.2059$ & $-8.3055\pm6.4592$ \\
$h=17$ & $0.9130\pm0.0184$ & $0.1087\pm0.2030$ & $-8.4868\pm6.6007$ \\
$h=18$ & $0.9115\pm0.0176$ & $0.0753\pm0.2007$ & $-8.5143\pm6.6106$ \\
$h=19$ & $0.9113\pm0.0166$ & $0.0426\pm0.1919$ & $-8.5835\pm6.6533$ \\
$h=20$ & $0.9099\pm0.0167$ & $0.0119\pm0.1864$ & $-8.7251\pm6.7523$ \\

\multicolumn{4}{c}{\textbf{Held-out Sawyer}} \\
$h=1$ & $0.9816\pm0.0051$ & $0.8782\pm0.0225$ & $0.6821\pm0.0550$ \\
$h=2$  & $0.9720\pm0.0074$ & $0.8258\pm0.0366$ & $0.1957\pm0.3168$ \\
$h=3$  & $0.9656\pm0.0094$ & $0.7963\pm0.0415$ & $-0.5956\pm0.9193$ \\
$h=4$  & $0.9585\pm0.0111$ & $0.7695\pm0.0414$ & $-1.7371\pm1.7133$ \\
$h=6$  & $0.9514\pm0.0148$ & $0.7307\pm0.0406$ & $-4.5519\pm3.5895$ \\
$h=7$  & $0.9484\pm0.0134$ & $0.7118\pm0.0404$ & $-5.9537\pm4.3826$ \\
$h=8$  & $0.9440\pm0.0145$ & $0.6940\pm0.0377$ & $-7.2737\pm5.0679$ \\
$h=9$  & $0.9408\pm0.0141$ & $0.6761\pm0.0360$ & $-8.4299\pm5.6460$ \\
$h=11$ & $0.9328\pm0.0132$ & $0.6383\pm0.0286$ & $-10.1638\pm6.4574$ \\
$h=12$ & $0.9296\pm0.0121$ & $0.6236\pm0.0295$ & $-10.6728\pm6.6503$ \\
$h=13$ & $0.9276\pm0.0138$ & $0.6120\pm0.0285$ & $-11.2198\pm6.9406$ \\
$h=14$ & $0.9242\pm0.0158$ & $0.5993\pm0.0270$ & $-11.6801\pm7.1710$ \\
$h=15$ & $0.9222\pm0.0163$ & $0.5854\pm0.0274$ & $-12.0939\pm7.4304$ \\
$h=16$ & $0.9194\pm0.0147$ & $0.5714\pm0.0288$ & $-12.3088\pm7.5915$ \\
$h=17$ & $0.9175\pm0.0141$ & $0.5614\pm0.0277$ & $-12.5487\pm7.7254$ \\
$h=18$ & $0.9139\pm0.0144$ & $0.5458\pm0.0269$ & $-12.8602\pm7.9320$ \\
$h=19$ & $0.9087\pm0.0147$ & $0.5312\pm0.0326$ & $-13.0918\pm8.1273$ \\
$h=20$ & $0.9028\pm0.0159$ & $0.5155\pm0.0351$ & $-13.3202\pm8.2978$ \\
\bottomrule
\end{tabular}
\end{table}
\paragraph{Results.}
The main difference lies in long-horizon stability. All models use matched encoder capacity and latent dimension and are frozen when evaluated on robots not observed during training. As depicted by Table~\ref{tab:robonet-transfer}, A-JEPA maintains rollout $R^2$ above $0.90$ at horizon $20$ on both Franka and Sawyer, whereas LeWM degrades substantially and PLDM becomes unstable over longer horizons. These results indicate that the action-conditioned latent dynamics learned by A-JEPA transfer more
robustly across changes in robot platform, morphology, and visual appearance.

\newpage
\section{Discussion I: LeWorldModel as a Constrained Instance}
\label{sec:discussion-lewm}

It is instructive to compare our general information-theoretic objective with
the recent LeWorldModel (LeWM) objective~\citep{maes2026leworldmodel}. LeWM
optimizes two terms: a next-representation prediction loss and an
isotropic-Gaussian regularizer,
\begin{equation}
\mathcal L_{\mathrm{LeWM}}
=
\underbrace{
\left\|
\hat{\mathbf z}_{t+1}-\mathbf z_{t+1}
\right\|_2^2
}_{\text{prediction}}
+
\lambda_{\mathrm L}
\underbrace{
\mathrm{SIGReg}(\mathbf Z)
}_{\text{anti-collapse}},
\label{eq:lewm-objective}
\end{equation}
where
$\hat{\mathbf z}_{t+1}
=
\mathrm{pred}_{\boldsymbol\phi}
(\mathbf z_t,\mathbf a_t)$.
SIGReg encourages the marginal representation distribution to match an
isotropic Gaussian through normality tests over random one-dimensional
projections. In comparison, our general objective is
\begin{equation}
\mathcal L_{\lambda}
=
\underbrace{
-\mathbb E
\left[
\log
p_{\boldsymbol\phi}
(
\mathbf z_{t+1}
\mid
\mathbf z_t,\mathbf a_t
)
\right]
}_{\text{transition predictiveness}}
-
\lambda
\underbrace{
H(\mathbf z_{t+1})
}_{\text{information preservation}}.
\label{eq:discussion-general-objective}
\end{equation}
The two terms in Eq.~\eqref{eq:lewm-objective} can be related directly to the
two principles in Eq.~\eqref{eq:discussion-general-objective}.

\paragraph{Prediction term.}
Consider a restricted transition model with fixed isotropic covariance,
\begin{equation}
p_{\boldsymbol\phi}
(
\mathbf z_{t+1}
\mid
\mathbf z_t,\mathbf a_t
)
=
\mathcal N
\left(
\mathbf z_{t+1};
\boldsymbol\mu_{\boldsymbol\phi}
(\mathbf z_t,\mathbf a_t),
\sigma^2\mathbf I
\right).
\end{equation}
Its negative log-likelihood is
\begin{equation}
-\log
p_{\boldsymbol\phi}
(
\mathbf z_{t+1}
\mid
\mathbf z_t,\mathbf a_t
)
=
\frac{1}{2\sigma^2}
\left\|
\mathbf z_{t+1}
-
\boldsymbol\mu_{\boldsymbol\phi}
(\mathbf z_t,\mathbf a_t)
\right\|_2^2
+
C,
\end{equation}
where $C$ is independent of the predicted mean. Hence, with fixed isotropic
variance, conditional likelihood maximization reduces exactly, up to scaling
and an additive constant, to the MSE prediction objective used by LeWM.

\paragraph{Information-preservation term.}
LeWM addresses representation collapse by encouraging
$\mathbf z\sim\mathcal N(\mathbf 0,\mathbf I)$ through SIGReg. This can also
be connected to the entropy term in our general objective. In particular,
among distributions satisfying
\begin{equation}
\mathbb E[\mathbf z]=\mathbf 0,
\qquad
\mathrm{Cov}(\mathbf z)=\mathbf I,
\end{equation}
the isotropic Gaussian is the unique maximum-entropy distribution,
\begin{equation}
\mathcal N(\mathbf 0,\mathbf I)
=
\arg\max_{p_{\mathbf z}}
H(\mathbf z).
\end{equation}
Therefore, Gaussian regularization can be viewed as a constrained realization
of the information-preservation principle in
Eq.~\eqref{eq:discussion-general-objective}: rather than maximizing entropy
over a general admissible representation family, LeWM explicitly drives the
representation toward the maximum-entropy distribution associated with fixed
zero mean and identity covariance.

Taken together, these observations show that LeWM can be viewed as a
constrained instance of our general information-theoretic formulation:
its MSE term corresponds to conditional likelihood maximization under a
fixed-isotropic Gaussian transition model, while SIGReg realizes entropy-based
information preservation under prescribed first- and second-order moments.
This interpretation concerns the underlying objective principle; SIGReg is not
algebraically identical to the entropy term itself.

\paragraph{What is missing for component-wise identification?}
This connection also clarifies why strong predictive performance and
non-collapsed representations do not by themselves imply recovery of individual
causal states. Under an orthogonal transformation $\mathbf R$,
\begin{equation}
\left\|
\mathbf R\hat{\mathbf z}
-
\mathbf R\mathbf z
\right\|_2^2
=
\left\|
\hat{\mathbf z}-\mathbf z
\right\|_2^2,
\end{equation}
and an isotropic Gaussian is likewise invariant under orthogonal rotations.
Thus, the LeWM objective does not, by itself, distinguish a latent
representation from its orthogonal mixtures. Our identifiability analysis
shows that the additional ingredient needed to resolve this ambiguity is
sufficient action-induced variation in the conditional transition mechanisms.
A-JEPA therefore extends this predictive and information-preserving view by
allowing actions to modulate both conditional means and variances, and uses
their variation as an identifying signal that reduces the ambiguity to
permutation and component-wise invertible transformations.

\section{Discussion II: Two Levels of Representation Collapse}
\label{sec:discussion-collapse}

Representation collapse is a central issue in joint-embedding predictive
learning. We argue that it is useful to distinguish two conceptually different
levels of representation failure. The first is the conventional
\emph{trivial collapse} widely studied in the JEPA and joint-embedding
literature, where the learned representation discards information about the
underlying state. The second is a more subtle failure that can remain even
after such information loss has been prevented: the representation may preserve
the underlying state information globally, while mixing multiple latent factors
within individual representation coordinates. We refer to this second level
more precisely as \emph{representation mixing}.

\paragraph{Level I: information-loss collapse.}
The conventional representation-collapse problem arises when distinct
observations are mapped to identical or nearly identical representations. In
the extreme case, $h(\mathbf x)=\mathbf c$ for all $\mathbf x$, such that the
predictive objective can be satisfied by a trivial solution without retaining
meaningful information about the underlying state. For an action-conditioned
predictive model, for example, if $\mathbf z_t=\mathbf z_{t+1}=\mathbf c$ for
all observations, a predictor that always outputs $\mathbf c$ can achieve
perfect representation-space prediction while learning no useful state
representation. This is fundamentally an information-loss problem: the
variation that distinguishes different underlying states has been discarded by
the encoder.

Existing joint-embedding methods address this first level through different
anti-collapse mechanisms. I-JEPA uses an asymmetric target-encoder architecture
to avoid trivial constant solutions, while VICReg explicitly maintains
representation variability and reduces redundancy through variance and
covariance regularization~\citep{assran2025v,bardes2021vicreg}. More
recently, LeWorldModel (LeWM) uses SIGReg to encourage the marginal
representation distribution to approach an isotropic Gaussian, thereby
preventing the representation from degenerating to a constant
solution~\citep{maes2026leworldmodel}. Although these approaches differ in
their implementation, they share the goal of preserving sufficient variation
in the learned representation.

Our general information-theoretic objective expresses this requirement
directly through marginal entropy maximization:
\begin{equation}
\mathcal L_{\lambda}
=
-\mathbb E
\left[
\log
p_{\boldsymbol\phi}
(\mathbf z_t\mid\mathbf z_{t-1},\mathbf a_{t-1})
\right]
-
\lambda H(\mathbf z_t).
\label{eq:discussion-collapse-objective}
\end{equation}
The entropy term directly opposes concentration of the representation
distribution toward a degenerate solution. In particular, under the bounded
representation space considered in Theorem~\ref{thm:optimal-solution}, every
global optimum with $\lambda>1$ attains the maximum possible entropy, with
$H(\mathbf z_t)=0$ and
$\mathbf z_t\sim\mathrm{Unif}((0,1)^d)$. Hence, the constant-representation
solution is explicitly excluded at the global optimum. Rather than prescribing
a particular architectural or distributional anti-collapse mechanism, our
formulation captures this requirement through a general
information-preservation principle.

LeWM provides a particularly direct connection to this interpretation. SIGReg
encourages $\mathbf z$ to follow $\mathcal N(\mathbf 0,\mathbf I)$. Under fixed
first- and second-order moments,
$\mathbb E[\mathbf z]=\mathbf 0$ and
$\mathrm{Cov}(\mathbf z)=\mathbf I$, the isotropic Gaussian is the unique
maximum-entropy distribution. Therefore, Gaussian distribution matching in
LeWM can be viewed as a constrained realization of the same
entropy-maximization principle. Our practical A-JEPA objective instead
implements this principle through the negative-sample contrastive term in
Eq.~\eqref{eq:gaussian-contrastive-loss}, whose connection to representation
entropy is established in Sec.~\ref{app:contrastive-entropy}.

\paragraph{Level II: information-preserving representation mixing.}
Preventing information-loss collapse does not, however, determine how the
retained information is organized across representation coordinates. Consider
a representation $\mathbf z=T(\mathbf s)$, where $T$ is an invertible
transformation of the underlying latent state $\mathbf s$. Since $T$ is
invertible, no state information need be lost: $\mathbf s$ can in principle be
fully recovered from $\mathbf z$. Nevertheless, each learned coordinate may
depend on multiple latent state variables, such that
$z_i=T_i(s_1,\ldots,s_d)$. The representation is therefore non-collapsed in
the conventional information-preservation sense, while the underlying latent
factors remain mixed across representation coordinates.

Importantly, entropy maximization alone cannot resolve this second ambiguity.
An invertible transformation can preserve the information contained in the
latent state while substantially mixing its individual components. More
directly, the change-of-variables formula gives
\begin{equation}
H(T(\mathbf s))
=
H(\mathbf s)
+
\mathbb E
\log
\left|
\det J_T(\mathbf s)
\right|.
\label{eq:discussion-entropy-transform}
\end{equation}
For a volume-preserving transformation satisfying
$|\det J_T(\mathbf s)|=1$, this reduces to
$H(T(\mathbf s))=H(\mathbf s)$. Hence, a fully mixed representation can have
exactly the same entropy as the original latent state.

A particularly simple example arises for isotropic Gaussian representations.
If $\mathbf s\sim\mathcal N(\mathbf 0,\mathbf I)$ and
$\mathbf z=\mathbf R\mathbf s$ for an orthogonal matrix $\mathbf R$, then
$\mathbf z\sim\mathcal N(\mathbf 0,\mathbf I)$ and therefore has exactly the
same entropy and marginal distribution as $\mathbf s$. Nevertheless, a general
orthogonal matrix $\mathbf R$ can mix every learned coordinate across multiple
underlying latent variables. Thus, neither information preservation nor
maximum entropy alone distinguishes the underlying causal coordinates from
their information-preserving mixtures.

Resolving this second level therefore requires additional identifying
structure rather than stronger anti-collapse regularization. In our framework,
this structure is provided by action-induced variation in the transition
mechanisms. As established in Theorem~\ref{thm:componentwise-identifiability}, sufficiently
rich action-induced variation rules out arbitrary cross-component mixing and
restricts the learned representation to permutation and component-wise
invertible transformations of the underlying latent causal states. In this
sense, the identifiability analysis addresses a fundamentally different problem
from conventional anti-collapse mechanisms: it determines how the preserved
state information is organized across individual representation coordinates.

The two levels therefore require different solutions. The first asks whether
the representation retains sufficient information to distinguish different
underlying states and can be addressed through anti-collapse mechanisms such
as architectural asymmetry, variance preservation, distribution matching, or
entropy maximization. The second asks whether the retained information is
organized according to the underlying latent factors rather than arbitrary
mixtures and requires additional identifiability conditions. In our
formulation, these roles are separated explicitly: \emph{entropy maximization
prevents information-loss collapse, whereas action-induced mechanism variation
resolves information-preserving representation mixing}.
\end{document}